\documentclass[pdflatex,sn-mathphys-ay]{sn-jnl}

\usepackage{graphicx}%
\usepackage{multirow}%
\usepackage{amsmath,amssymb,amsfonts}%
\usepackage{amsthm}%
\usepackage{mathrsfs}%
\usepackage[title]{appendix}%
\usepackage{xcolor}%
\usepackage{textcomp}%
\usepackage{manyfoot}%
\usepackage{booktabs}%
\usepackage{algorithm}%
\usepackage{algorithmicx}%
\usepackage{algpseudocode}%
\usepackage{listings}%
\usepackage{makecell}
\usepackage{setspace}
\usepackage{newtxtext,newtxmath}

\theoremstyle{thmstyleone}%
\newtheorem{theorem}{Theorem}
\newtheorem{assumption}[theorem]{Assumption}
\newtheorem{lemma}[theorem]{Lemma}
\newtheorem{corollary}[theorem]{Corollary}

\theoremstyle{thmstyletwo}%
\newtheorem{remark}{Remark}%

\theoremstyle{thmstylethree}%

\newcommand{\pr}{\mathbb{P}}
\newcommand{\E}{\mathbb{E}}
\newcommand{\V}{\mathbb{V}}
\newcommand{\var}{\mathrm{Var}}

\def\beqlb{\begin{eqnarray}}\def\eeqlb{\end{eqnarray}}
\def\beqnn{\begin{eqnarray*}}\def\eeqnn{\end{eqnarray*}}

\begin{document}

\title[Article Title]{Improved Model-based Reinforcement Learning with Smooth Kernels}


\author[1]{\fnm{Kun} \sur{Long}}\email{52254404001@stu.ecnu.edu.cn}

\author*[1]{\fnm{Yuqiang} \sur{Li}}\email{yqli@stat.ecnu.edu.cn}

\author*[1]{\fnm{Xianyi} \sur{Wu}}\email{xywu@stat.ecnu.edu.cn}

\affil[1]{\orgdiv{School of Statistics, KLATASDS-MOE}, \orgname{East China Normal University}, \orgaddress{\street{North Zhongshan Road}, \city{Shanghai}, \postcode{200062}, \country{PR China}}}


\abstract{For continuous state-action space scenarios, classical reinforcement learning (RL) theory predominantly focuses on low-rank Markov decision processes (MDPs), which provide sample-efficient guarantees at the expense of restrictive structural assumptions. Kernel smoothing model-based approaches offer a promising alternative paradigm that instead leverages the smoothness of the MDP and employs non-parametric kernel smoothing estimates of transition dynamics.
This paper proposes a new kernel-smoothing model-based approach for online reinforcement learning in finite-horizon settings under Lipschitz continuity assumptions on the MDP. By incorporating a Bernstein-style exploration bonus into the kernel smoothing framework, our method achieves a regret bound of $\widetilde{O}\left(H^{\frac{d_{{SA}}+3}{d_{{SA}}+2}}K^{\frac{d_{{SA}}+d_{{S}}}{d_{{SA}}+d_{{S}}+1}}\right)$, where $H$ is the horizon length, $K$ is the number of episodes, $d{_{SA}}$ is the covering dimension of the joint state–action space, and $d_{S}$ is the covering dimension of the state space. This result improves upon the state-of-the-art regret bound in its dependence on $H$. The theoretical advancement relies on a delicate analysis of the synergy between Bernstein-style bonuses and kernel smoothing, where a new tight Bernstein-type concentration inequality for martingales may be of independent interest.}

\keywords{Reinforcement learning, Regret bound, kernel smoothing, Bernstein's inequality}



\maketitle

\section{Introduction}
Recent decades have witnessed  significant advancements in RL algorithms, with substantial work devoted to understanding the theoretical foundations of RL, especially on the tabular cases for online or offline RL (\cite{sutton2018reinforcement, vamvoudakis2021handbook}).

For online RL, the critical challenge of dynamically balancing Exploration-Exploitation is addressed by, e.g. upper confidence bound value iteration (UCBVI) algorithm (\cite{azar2017minimax}), if the state and action spaces are both finite.


However, extensive practice such as autonomous navigation and robotic control involve continuous state-action pairs. In such regimes, traditional model-based approaches relying on frequency counts for transition estimation do not work, because almost all state-action pairs are never visited throughout the whole learning process.

This challenge can only be addressed when the models are subject to certain smoothness properties, so that one can estimate the value functions by means of average over spatial proximity. For these smoothing models, a type of  promising paradigm, known as\textit{ kernel smoothing} techniques \citep{ormoneit2002kernel}, have emerged. These techniques, instead of explicitly modeling transitions, generalize knowledge across similar states and actions by weighting historical experiences based on spatial proximity. Subsequent innovations continued in refining this approach  can be found in \citet{jong2006kernel, kveton2012kernel, kveton2013structured, barreto2016practical, lim2019kernel}, among others. Here we would like to note that the term kernel here is not the one in the so-called \textit{kernel tricks} through Reproducing Kernel Hilbert Spaces.


The state-of-the-art regret bound for kernel-smoothing model-based reinforcement learning in continuous MDPs was established recently by \citet{Domingues2020KernelBasedRL}. Their Kernel-UCBVI algorithm uses kernel smoothing to estimate the transition model and adapts the UCBVI framework to compute value functions. By employing a Hoeffding-type exploration bonus that assigns confidence widths uniformly across all state-action pairs, regardless of transition variance, they obtained a regret bound of $\widetilde{O}\!\left(H^3K^{\frac{2d_{SA}}{2d_{SA}}+1}\right)$.
\citet{Domingues2020KernelBasedRL} observed that this bound contains an additional $\sqrt{H}$ factor compared with adaptive Q-learning methods that adaptively partition the continuous state-action space \citep{Sinclair2019AdaptiveDF}, although their experiments suggest that model-based methods can perform better in practice. They therefore left open whether this gap is due to the algorithm design or merely the analysis. More broadly, in terms of the horizon dependence, existing kernel-smoothing model-based methods remain far from the conjectured lower bound of order $H$, and are also worse than the $H^{1+\frac{1}{d_{SA}+1}}$ dependence achieved by the model-based AdaMB algorithm \citep{Sinclair2020AdaptiveDF}.

This naturally raises the question of whether one can design a model-based algorithm with kernel smoothing that achieves a better regret dependence on the horizon $H$. Motivated by this question, we study online finite-horizon reinforcement learning in inhomogeneous MDPs under essentially the same setting as \citet{Domingues2020KernelBasedRL}, where the transition model is estimated via kernel smoothing, and establish finite-sample regret guarantees.

To improve the dependence on $H$, a natural idea is to incorporate variance-induced exploration. In tabular MDPs, Bernstein-type bonuses yield sharper regret bounds in terms of $H$ than Hoeffding-type bonuses \citep{azar2017minimax,Zhang2020IsRL}. The main challenge, however, is how to construct such Bernstein-type bonuses in the kernel-smoothing setting. In the tabular case, transition estimators are formed from frequency counts of i.i.d.\ samples, so empirical Bernstein inequalities apply directly \citep{Audibert2009}. In contrast, kernel smoothing aggregates all collected samples, thereby destroying the i.i.d.\ structure underlying such arguments. As a result, existing empirical Bernstein inequalities do not directly apply in the kernel-smoothing setting. To overcome this difficulty, we establish a new martingale-based empirical Bernstein inequality that does not rely on the i.i.d.\ assumption. Based on this result, we propose a Kernel-Based Value Iteration with Bernstein-type UCB algorithm, referred to as KBVI-BUCB, which achieves an improved regret dependence on $H$.


\begin{table}[t]
    \centering
    \scriptsize
     \caption{Comparisons of KBVI-BUCB with several regret bounds for RL in continuous settings,  $K$ is the number of episodes and $H$ is the number of planning horizon.  Uniform Reward means "$r_h\in[0, 1]$ for all $h$."}
    \label{table1}
    \begin{tabular}{c|c|c}
        \toprule[2pt]
        Paper & Regret & Uniform Reward \\
        \midrule[2pt]
        \makecell{Kernel-UCBVI\\\citep{Domingues2020KernelBasedRL}}
        &  $\widetilde{O}\big(H^{3}K^{\frac{2d_{_{SA}}}{2d_{_{SA}} + 1}}\big)$ & Yes \\
        \makecell{AdaMB\\\citep{Sinclair2020AdaptiveDF}}
         & \makecell{$\widetilde{O}\big(H^\frac{d_{_{SA}} + 2}{d_{_{SA}} + 1}K^{\frac{d_{_{SA}} + d_{\mathcal{S}} - 1}{d_{_{SA}} + d_{\mathcal{S}}}}\big)$ ($d_{\mathcal{S}} > 2$)\\$\widetilde{O}\big(H^{\frac{d_{_{SA}} + 2}{d_{_{SA}} + 1}}K^{\frac{d_{_{SA}} + d_{\mathcal{S}} + 1}{d_{_{SA}} + d_{\mathcal{S}} + 2}}\big)$ ($d_{\mathcal{S}} \leq 2$)} & Yes \\
         \makecell{Adaptive Q-learning\\\citep{Sinclair2019AdaptiveDF}}
         & $\widetilde{O}\big(H^{5/2}K^{1-\frac{1}{d_{_{SA}}+2}}\big)$& Yes\\
         \makecell{Net-Based Q-learning\\\citep{Song2019EfficientMR}}&  $\widetilde{O}\big(H^{5/2}K^{1-\frac{1}{d_{_{SA}}+2}}\big)$& Yes\\
         This paper &  $\widetilde{O}\big(H^{\frac{d_{_{SA}+3}}{d_{_{SA}+2}}}K^{\frac{d_{_{SA}} + d_{_S}}{d_{_{SA}} + d_{_S}+1}}\big)$& No\\
        \bottomrule[2pt]
    \end{tabular}
    \label{table}
\end{table}

The main contributions are summarized as follows:
\begin{enumerate}
    \item We demonstrate that the kernel smoothing approach plus the upper confidence bound value iteration (VI) framework can obtain near optimal regrets.  It turns out that the KBVI-BUCB can achieve a regret bound
\(
    \widetilde{O}\big(H^{\frac{d_{_{SA}}+3}{d_{_{SA}}+2}}K^{\frac{d_{_{SA}} + d_{_S}}{d_{_{SA}} + d_{_S} + 1}}\big),
\)
 even improving upon the state-of-the-art result in the model based RL with continuous state-action spaces by a factor $H^{\frac{1}{(d_{_{SA}}+2)(d_{_{SA}}+1)}}$.
 As illustrated in Table \ref{table1},  the KBVI-BUCB improves upon the bounds in \cite{Domingues2020KernelBasedRL} by a factor near $H^2$ and is better than the results of adaptive Q-learning in \cite{Sinclair2019AdaptiveDF} by a factor near $H^{3/2}$. These theoretical improvements provide an answer for the open question mentioned in \cite{Domingues2020KernelBasedRL}. Compared to \cite{Sinclair2020AdaptiveDF}, our results improve not only on $H$ but also on $K$ when $d_{_S}\leq 2$.

    \item In developing the theoretical finite-sample guarantees, we establish a new Bernstein concentration inequality for martingales (Theorem \ref{aux.4}),
    which is more applicable and is also of independent interest. 
While recent advances in Bernstein inequalities, e.g. \citet{whitehouse2024modern}, provide general self-normalized bounds and can be broadly applicable and
 powerful, Theorem \ref{aux.4} here provides  a more direct and practically applicable form of the empirical Bernstein inequality tailored specifically to martingale difference sequences. This makes the bounds depending directly on the empirical
 quadratic term $\sum_{i=1}^n X_i^2$,  a form that is particularly convenient for constructing confidence intervals and
 deriving regret bounds in online learning. Moreover, it appears that we provide the first rigorous proof to this type of Bernstein inequality, see the first paragraph of Section \ref{proof_lem_2}.
     \item Finally, while uniform per-step reward bounds $0 \leq r_h \leq 1$ are typically required in existing kernel-based RL, the current work proceeds  only under $\sum_{h=1}^Hr_h \leq H$  the results of \cite{Jiang2018OpenPT}, who studied the dependence of sample complexity lower bounds on planning horizon, can be invoked to provide a fair environment when comparing our results to the existing ones.  This assumption generalizes bounded reward models to include sparse-reward scenarios.
\end{enumerate}

This paper is structured as follows: Section \ref{pre} is a preliminary section that presents  necessary knowledge of MDPs. Section \ref{main_alg} presents the KBVI-BUCB algorithm, integrating kernel-based VI with UCB bonuses. Section \ref{theo_gua} establishes non-asymptotic regret bounds, followed by a proof sketch in Section \ref{tech_over}.
Section \ref{expe} reports some experimental results and the final section concludes the paper. The details of the proof and more literature related to our work are arranged in the appendix.

\section{Related works}
\label{rela_wor}
The prior works that are relevant to our study reside at the intersection of the key areas Optimism in online RL, pessimism in offline RL, kernel smoothing, parametric models including RKHS, and adaptive partition. Below, we highlight several of the most pertinent studies in the field. 
\bmhead{Optimism in online RL} Our work belongs to the online learning framework that focuses on maximizing cumulative rewards through interactions with unknown environments. The primary challenge in this domain is balancing exploration and exploitation. One widely adopted approach is the principle of optimism, typically in the form of upper confidence bounds (UCB). For episodic tabular MDPs, the minimax lower bound on regret is $\widetilde{\Omega}(\sqrt{H^3|\mathcal{S}||\mathcal{A}|K})$ \citep{azar2017minimax, jin2018q}. \citet{azar2017minimax} introduced an efficient algorithm UCBVI, which, based on Bernstein and Freedman inequalities, adds a UCB bonus directly to the optimal Q-functions. It achieves the lower bound when $K \geq H^2|\mathcal{S}|^3|\mathcal{A}|$ and $|\mathcal{S}||\mathcal{A}| \geq H$. More recently, \citet{Zhang2020IsRL} revealed that the regret bound for time-homogeneous MDPs allows for at most a polylogarithmic dependence on  $H$ provided that the total reward was bounded in $[0,1]$.  They developed an efficient algorithm called MVP that achieves the lower bound when $K \geq |\mathcal{S}|^3|\mathcal{A}|$. \citet{zhang2024settling} revisited MVP and demonstrated that this algorithm matched the minimax lower bound for every sample size $K \geq 1$. All of these approaches are model-based.

The principle of optimism has also been applied to model-free algorithms, such as $Q$-learning \citet{jin2018q}. However, no model-free algorithm has matched the lower bound. Moreover, the principle of optimism has also been explored in  infinite horizon MDPs \citet{ortner2013adaptive, Lakshmanan2015ImprovedRB, Qian2019ExplorationBF} as well as in continuous MDPs \citet{Song2019EfficientMR, Sinclair2019AdaptiveDF, Domingues2020KernelBasedRL, zhou2022computationally}.

\bmhead{Pessimism in Offline RL} In contrast, all offline learning algorithms refrain from executing additional data collection during the learning process.
The target of offline RL problems is to learn a near-optimal policy from pre-collected data. The challenges of offline RL primarily stem from  distributional shift and/or data coverage. The former refers to the potential disparity between
the policy used for data collection and the optimal policies. The latter indicates that some states and
actions may have abundant data, while others may have very little. In order to address the challenges, a line of works put forward the pessimism principle, see, e.g., \citet{Chen2021PessimismMI, Xie2021PolicyFB,Shi2022PessimisticQF,Rashidinejad2022Bridging,Yan2023The,Li2024Settling}. The first asymptotic sample complexity bound for tabular finite-horizon MDPs was established as $O\big(\frac{H^6|\mathcal{S}|C^*}{\varepsilon^2}\big)$ (where $\varepsilon$ is the target sub-optimality and $C^*$ single-policy concentrability coefficient) in \citet{Rashidinejad2022Bridging} and \citet{Xie2021PolicyFB} by incorporating Hoeffding-style lower confidence bounds into value iteration. Such a sample complexity bound was a large factor of $H^2$ above the minimax lower bound. More recently, \citet{Li2024Settling} proposed the state-of-the-art algorithm that achieved sample complexity bounds of the form $O\big(\frac{H^4|\mathcal{S}|C^*_{\text{clipped}}}{\varepsilon^2}\big)$, which was minimax optimal. 

\bmhead{Kernel smoothing in RL} Kernel smoothing is an extensively used standard nonparametric method and has a long history in the community of statistics. \citet{ormoneit2002kernel} appeared to be the first to introduce kernel methods into RL. They designed a Kernel Based RL (KBRL) algorithm on an i.i.d. sample from transition distribution, and then estimated  optimal value function with asymptotic convergence.   Combining strengths of KBRL and model-based planning, \citet{jong2006kernel} proposed an  algorithm named KBPS for online RL problems which converges to the optimal policy more rapidly than prior works. \citet{kveton2012kernel} developed a method that approximates KBRL in certain representative states and has lower time complexity than KBRL. \citet{barreto2016practical} proposed a stochastic factorization technique to reduce computational complexity and proposed a KBSF (kernel-based stochastic factorization) algorithm. The computational complexity of KBSF was linear in the number of sample transitions and KBSF was empirically shown on real-world data to outperform some previous approaches. \citet{Domingues2020KernelBasedRL} introduced Kernel-UCBVI and provided the first regret bound for episodic MDPs within the kernel smoothing framework. \citet{Lakshmanan2015ImprovedRB}  considered the problem of undiscounted RL, employed kernel density estimators, and provided a regret bound for undiscounted infinite horizon continuous MDPs. Other applications of kernel smoothing  in RL include, e.g., \citet{lim2019kernel} under the robust MDP framework and \citet{domingues2021kernel} on non-stationary models where the transition and reward functions vary across episodes.

\bmhead{Parametric assumption of nontabular MDP} There are several related papers addressing continuous MDP problems without relying on kernel smoothing techniques. They had tackled the challenge of determining the sample complexity for continuous MDPs under parametric assumptions, such as linear models \cite{yang2019sample, jin2020provably, he2023nearly, agarwal2023vo, zhang2024horizon}, reproducing kernel Hilbert spaces (RKHS) \cite{kakade2020information, yang2020function, chowdhury2023value, vakili2023kernelized, yeh2023sample, lai2024leveraging, vakili2024reward, kayal2025near}, the linear mixture models \cite{ayoub2020model,zhang2021variance} and so on. 	
 Among them,
 \cite{jin2020provably}  derived a regret bound of  $\widetilde{O}(\sqrt{d^3H^3K})$ which was later improved by \cite{he2023nearly} and \cite{agarwal2023vo} to $\widetilde{O}(\sqrt{d^2H^3K})$. More recently,  \citet{zhang2024horizon} obtained the first horizon-free regret bound of $\widetilde{O}(\sqrt{d^{11}K})$. 
\citet{kakade2020information} studied a nonlinear control problem where the transition dynamics lie in the RKHS of a known kernel. 
\citet{yang2020function} assumed that the Bellman optimality operator maps any bounded state-action value function into an RKHS with bounded norm, and developed an optimistic least-squares value iteration algorithm. They obtained a regret of $\widetilde{O}(H^2\sqrt{K})$ for kernels with exponential decay Mercer eigenvalues, such as squared exponential kernel. For kernels with polynomially decaying eigenvalues,  \citet{vakili2023kernelized} assumed the transition model could be represented using the RKHS of a known kernel function, and achieved a sublinear regret of $\widetilde{O}(H^2K^{\frac{d+\alpha/2}{d+\alpha}})$ with smooth parameter $\alpha$, which was the first sublinear regret bound under this setting. Notably, with Mat\'{e}rn kernels, their result matches the lower bound given by \citet{scarlett2017lower}. More recently, \citet{lai2024leveraging, vakili2024reward} and \citet{kayal2025near} explored reward-free RL within the RKHS framework.
\citet{ayoub2020model} considered the linear mixture model, which posits that the transition dynamics are a linear combination of known base models and achieved a regret bound of $\widetilde{O}(\sqrt{d^2H^4K})$ and later, \citet{zhang2021variance} derived the first horizon-free regret bound for this setting. 

	\bmhead{Adaptive Partitions} Another recent line of research on RL problems with continuous state-action spaces,  see, e.g., \citet{Song2019EfficientMR, Sinclair2019AdaptiveDF} and \citet{Sinclair2020AdaptiveDF}),  use discretization of the continuous spaces. Nevertheless, this approach often suffered from impractical computational complexity, especially in high-dimensional settings.

\section{Problem formulation}
\label{pre}
 We will use the following notation to simplify the presentation: $[n]:=\{1, 2, \cdots, n\}$ for any positive integer $n$,  $\mathbb{R}_{\geq 0}:=[0,\infty)$, $\mathbf{1}[\mathcal{E}] (\hbox{or }\mathbf{1}_{\mathcal{E}}(x))$ the indicator of $\mathcal{E}$  and  $Pf := \int f dP$ for any function $f$ and measure $P$.

An $H$-horizon Markov decision process is  a tuple $\mathcal{M} = (\mathcal{S}, \mathcal{A}, P, r, H)$ described as follows:

(1). $\mathcal{S}$ and $\mathcal{A}$ are the  state and action spaces, equipped with metrics $\rho_{\mathcal{S}}$ and $\rho_{\mathcal{A}}$  respectively, both of which are compact. A metric $\rho:=\rho_{\mathcal{S}}\times\rho_{\mathcal{A}}$ defined by
\(
  \rho[(s, a), (s', a')] = \rho_{\mathcal{S}}(s, s') + \rho_{\mathcal{A}}(a, a')
\) is also induced for  the product  space $\mathcal{S} \times \mathcal{A}$.
Denote also by $\mathcal{T}_{\mathcal{S}}$ and $\mathcal{T}_{\mathcal{A}}$ the Borel $\sigma$-algebras induced by the metrics $\rho_{\mathcal{S}}$ and $\rho_{\mathcal{A}}$, respectively.

(2). $P = \left\{P_h : \mathcal{S} \times \mathcal{A} \times \mathcal{T}_{\mathcal{S}} \to [0, 1]\right\}_{h\in [H]}$ is the transition kernel. For any $(s, a) \in \mathcal{S} \times \mathcal{A}$, write $P_hf(s, a) = \int f(x)P_h(dx|s, a)$  and $\mathbb{V}_h(f, s, a) := P_h[f - P_hf(s, a)]^2(s, a)$.

 (3). $r = \{r_h : \mathcal{S} \times \mathcal{A} \to \mathbb{R}_{\geq 0}\}_{h\in [H]}$ is a $H$-dimensional vector of specified reward functions, satisfying \begin{equation}\label{assu.1}r_h \geq 0 \hbox{ and }\sum_{h = 1}^Hr_h(s_h, a_h) \leq H, h \in [H] \hbox{ almost surely}
\end{equation}  under every policy.

We assume specified reward functions to simplify the analysis, as in many existing works e.g., \citet{Zhang2022HorizonFreeRL, Sinclair2019AdaptiveDF}. Moreover, while most of the existing work took $r_h \in [0, 1]$ for all $h \in [H]$,   \citet{Jiang2018OpenPT} took $\sum_{h = 1}^Hr_h \leq 1$. Eq \eqref{assu.1} takes the form of the latter to extend the former, so as to also accommodate sparse-rewards where  significant rewards may occur at critical steps ($r_h > 1$ permitted) and zero rewards can be collected otherwise.

 We work with deterministic policies $\pi = \{\pi_h\}_{h = 1}^{H}$, where $\pi_h$ is a map $\pi_h : \mathcal{S} \to \mathcal{A}$.  Given an initial state $s_1$ and a policy $\pi$, a trajectory $(s_1, a_1, r_1, s_2, \cdots, s_H, a_H, r_H, s_{H + 1})$ is generated following the rule $a_h \sim \pi_h(s_h)$, $r_h = r_h(s_h, a_h), s_{h + 1} \sim P_{h}(\cdot|s_h, a_h)$.
Denote by
\beqnn
    V^\pi_h(s)&=&\E\left[\sum_{t = h}^Hr_{t}\Big|s_h = s, \pi\right],
    \\ Q^\pi_h(s, a)&=&\E\left[\sum_{t = h}^Hr_{t}\Big|s_h = s, a_h = a, \pi\right]
\eeqnn
the value functions at a state $s$ and a state-action pair $(s, a)$ at step $h$, respectively, for a policy $\pi$. Clearly, $V_{H + 1}^\pi(s) = 0$ for all $s \in \mathcal{S}$ and all policy $\pi$. Denote {$V^\pi_1(s), Q^\pi_1(s, a)$ by $V^\pi(s), Q^\pi(s, a)$ and further}
\(
    V^*_h(s) = \sup_{\pi}V^{\pi}_h(s)\) and  \(Q^*_h(s, a) = \sup_{\pi}Q_h^\pi(s, a)
\)
the optimal value functions.
We are concerned with an online setting in which the agent updates its policy at the end of every episode,  taking into account the most updated experience collected and then interacts with the environment with the new policy. Denote by $\pi^k$ the policy by which the agent interacts with the environment through the $k$-th episode. Then the objective is to bound the regret
\begin{equation}\label{regretdefinition}
    \mathcal{R}(K) = \sum_{k = 1}^K\Big(V_1^*(s_1^k) - V_1^{\pi^k}(s_1^k)\Big).
\end{equation}
 of the policy $\pi=\{\pi_k: {k\in[K]}\}$.

For continuous state and action spaces, the problem can only be attacked under certain continuity conditions. The following Lipschitz assumption on the model is fundamental for our analysis.\begin{assumption}[Lipschitz property]
    \label{assu.2}
    There exist constants $\lambda_r>0$ and $\lambda_p\in(0, 1)$ such that, for all $h\in[H]$,
    \beqnn
 	       |r_h(s, a) - r_h(s', a')| &\leq&\lambda_{r}\rho[(s, a), (s', a')],\\
	       W_{1}(P_h(\cdot|s, a), P_h(\cdot|s', a'))&\leq&\lambda_{p}\rho[(s, a), (s', a')],
    \eeqnn
where  $$W_{1}(\mu, \nu) = \sup_{f : Lip(f) \leq 1}\int_{\mathcal{S}}f(y)(d\mu(y) - d\nu(y))$$ is the 1-Wasserstein distance  between measures $\mu$ and $\nu$, where $Lip(f)$ is the Lipschitz constant of function $f$ w.r.t. $\rho_{\mathcal{S}}$.
\end{assumption}


Under the Lipschitz assumption of the model, the value functions are also Lipschitz with respect to the metrics:
\begin{lemma}[Lemma 1 of \citet{Domingues2020KernelBasedRL}]
    \label{lem.1}
Under Assumption \ref{assu.2}, $$|Q^*_h(s, a) - Q^*_h(s', a')| \leq L_{h}\rho[(s, a), (s', a')]$$ for all $h \in [H]$, where $L_{h} = \sum_{h' = h}^{H}\lambda_{r}\lambda_{p}^{H - h'}.$ As a result,  $V^*$ is also Lipschitz.
\end{lemma}
 Moreover,  a regret generally depends on $L_1$ that is exponential in $H$ from Lemma \ref{lem.1}. Therefore,   in Assumption \ref{assu.2}, the reason to take $\lambda_p < 1$ is  such that $L_1\leq \frac{\lambda_r}{1-\lambda_p}$, independent with $H$. 


 \section{Algorithm}
\label{main_alg} This section describes the KBVI-BUCB algorithm.
\subsection{Algorithm}
Basically, in each episode $k$, the agent follows the greedy policy $\pi^k$ with respect to the current estimates $Q_h^k$ of the $Q$-functions at $h\in[H]$, which will be depicted later on in detail in Section \ref{kernel-estimate}. The agent thus collects a trajectory $\tau^k=\{s_1^k,a_1^k,s_2^k,a_2^k,\dots, s_H^k,a_H^k,s_{H+1}^k\}.$ At the end of this epoch, the agent refines its transition kernel estimate with the accumulated observations and then recalculates the value functions using the updated transition kernel as illustrated in  Section \ref{kernel-estimate}. Pseudocode of the algorithm is presented in Algorithm \ref{alg.1}.

Note that it is challenging to exactly identify the optimal action in Eq \eqref{greedy-a} due to the infinity  of actions. A remedy  is to approximate the action space with a finite one by means of discretization. Unlike the state space, discretizing the action space is often more acceptable, as it tends to be simpler in many practical applications like autonomous driving. The Deterministic Optimistic Optimization algorithm by \citet{Munos2014FromBT} is one method to explore and approximate the optimal action efficiently.

\begin{algorithm}
   \caption{KBVI-BUCB}
   \label{alg.1}
\begin{algorithmic}[1]
   \Require Parameters $K, H, \beta, \delta, \lambda_p, \lambda_r, \sigma$
   \State $V^1_h(s) \Leftarrow H,
   Q_h^{1}(s, a) \Leftarrow H$ and $\mathcal{D}_h \Leftarrow \emptyset$ for all $h \in [H]$
   \For{$k \in [K]$}
   \State Obtain initial state $s^k_1$
   \For{$h \in [H]$}
   \State Take action
   \begin{equation}\label{greedy-a}a_h^k = \arg\max_{a \in \mathcal{A}}Q_h^k(s_h^k, a)
   \end{equation}
   \State Observe next state $s_{h + 1}^{k}$
   \State Add $(s_h^k, a_h^k, s_{h+1}^k)$ to $\mathcal{D}_{h}$
   \EndFor
   \State $\{Q^{k + 1}_h\}_{h\in[H]}$ = Estimate Q($k + 1, \{\mathcal{D}_{h}\}_{h\in[H]})$
   \EndFor
\end{algorithmic}
\end{algorithm}

\subsection{Nadaraya-Watson Kernel Estimates of the value functions}\label{kernel-estimate}

Depicted in the following steps is a Nadaraya-Watson procedure to estimate for every $k\in [K]$ the value functions using the sample $$\{s_1^l,a_1^l,s_2^l,a_2^l,\dots, s_H^l,a_H^l,s_{H+1}^l\}: l\in[k-1]$$

\bmhead{1. Nadaraya-Watson estimate of the transition kernel} Fix a kernel function $g : \mathbb{R}_{\geq 0} \to [0, 1]$. 
For any $(s, a) \in \mathcal{S} \times \mathcal{A}$ at $h \in [H]$,  define
\beqnn
	\widetilde{\mathnormal{w}}_h^{k, l}(s, a) =        \frac{w_h^l(s, a)}{\beta + \sum_{n = 1}^{k-1}w_h^n(s, a)}, l=1,2,\dots,k-1,
\eeqnn
 where $        \mathnormal{w}_h^l(s, a) = g({\rho((s, a), (s_h^l, a_h^l))\over\sigma})$, $\sigma$ is the bandwidth and  $0 < \beta \leq 1$ is a regularization parameter. With $\widetilde{\mathnormal{w}}^{k, l}_h(s, a)$ as weights, the transition distribution is estimated  by
\[
    \widehat{P}^k_h(\mathcal{O}|s, a) = \sum_{l = 1}^{k-1}\widetilde{\mathnormal{w}}^{k, l}_h(s, a)\mathbf{1}_{\mathcal{O}}(s_{h+1}^l),\quad\mathcal{O}\in \mathcal{T}_{\mathcal{S}}.
\]
It in fact defines a (deficit) discrete distribution putting masses $\widetilde{\mathnormal{w}}^{k, l}_h(s, a)$ at points $s_{h+1}^l$, $l=1,2,\dots, k-1$. Consequently, the conditional expectation  and variance of a function $f$ with respect to  $\widehat{P}^k_h(\cdot|s, a)$ are approximately $$\widehat{P}^k_hf(s, a) = \sum_{l=1}^{k-1}\widetilde{\mathnormal{w}}^{k, l}_h(s, a)f(s_{h+1}^l)$$ and $$\widehat{\V}^k_h(f,s, a) = \sum_{l=1}^{k-1}\widetilde{\mathnormal{w}}^{k, l}_h(s, a)\big[f(s_{h+1}^l) - \widehat{P}^k_hf(s, a)\big]^2.$$
\bmhead{2. Bonus} We define a Bernstein-style bonus function that takes the approximate form
\begin{align}
    \label{est.5}
    b_{h}^k(s, a) \approx \sqrt{\frac{\widehat{\mathbb{V}}^k_h(V_{h+1}^k,\, s, a)}{C^k_h(s, a)}} + \frac{H}{C^k_h(s, a)} + \sigma,
\end{align}
where $C^k_h(s, a) = \beta + \sum_{l=1}^{k-1}w^{l}_h(s, a)$ is a generalized count of $(s, a)$ and $\sigma$ accounts for the approximation bias introduced by kernel smoothing. The precise definition with all numerical constants and kernel-dependent quantities is deferred to Appendix \ref{appendix1}. The first two terms mirror the classical Bernstein bonus used in tabular settings: a variance-dependent term scaled by the generalized count and a range-dependent term of order $H/C^k_h(s,a)$. The last term $\sigma$ captures the bias inherent in smoothing over a continuous state--action space.

\bmhead{3. Estimate value functions}Construct an upper confidence bound (UCB) estimate
\[
\widetilde{Q}_h^k(s, a) = r_h(s, a) + \widehat{P}^k_hV_{h+1}^k(s, a) + b_h^k(s, a)
\]
of the $Q$ function. Because the true optimal $Q_h^*$ is $L_h$-Lipschitz (Lemma \ref{lem.1}), we can define an $L_h$-Lipschitz estimate of the $Q$ function using interpolation by
\begin{equation}
	Q_h^k(s, a) = \min_{l \in [k-1]}\big\{\widetilde{Q}_h^k(s_h^l, a_h^l) + L_h\rho[(s, a), (s_h^l, a_h^l)]\big\}\notag
\end{equation}
and a corresponding estimate of the value function $V$ by
$V_{h}^k(s) =\min \{\max_{a}Q_{h}^k(s, a), H\}.$

Presented in Algorithm \ref{alg.2} is a pseudocode of the procedure above.

\begin{algorithm}[H]
   \caption{Estimate $Q^k$}
   \label{alg.2}
\begin{algorithmic}
   \Require Episode indicator $k$ and dataset $\{\mathcal{D}_h\}_{h\in [H]}$
   \State $V_{H + 1}^k(s) \Leftarrow 0$ for all $s$.
   \For{$h = H, H - 1, \cdots, 1$}
   \For{$m \in [k-1]$}
   \For{$l \in [k-1]$}
   \State Set $ w_h^{l}(s_h^m, a_h^m) = g\big({\rho((s_h^m, a_h^m), (s_h^{l}, a_h^{l}))\over{\sigma}}\big)$
   \State Set $\widetilde{w}_h^{k, l}(s_h^m, a_h^m) = \frac{w_h^l(s_h^m, a_h^m)}{\beta + \sum_{n=1}^{k-1}w_h^n(s_h^m, a_h^m)}$
   \EndFor
   \State Calculate $b_h^k(s_h^m, a_h^m)$ as (\ref{est.5})
   \State $\widetilde{Q}_h^k(s_h^m, a_h^m) = \sum_{l = 1}^{k-1}\widetilde{\mathnormal{w}}_h^{k, l}(s_h^m, a_h^m)V_{h+1}^k(s_{h+1}^l) +r_h(s_h^m, a_h^m) + b_h^k(s_h^m, a_h^m)$
   \EndFor
   \State $Q_h^k(s, a)=\min\limits_{l \in [k-1]}\big\{\widetilde{Q}_h^k(s_h^l, a_h^l)+L_h\rho[(s, a), (s_h^l, a_h^l)]\big\}$
   \For{$l \in [k-1]$}
   \State $V_{h}^k(s_h^l) = \min\big\{\max\limits_a Q_{h}^{k}(s_h^l, a), H\big\}$
   \EndFor
   \EndFor
   \State \textbf{Return} $\{Q_{h}^k\}_{h \in [H]}$
   \end{algorithmic}
\end{algorithm}

\section{Theoretical guarantee}
\label{theo_gua}
This section  presents a theoretical guarantee of the algorithm above in terms of high-probability upper bounds under  certain assumptions on the model and kernel functions. The following is for the kernel functions, which is satisfied by the Gaussian kernel.
\begin{assumption}
     \label{assu.3}
     The function $g : \mathbb{R}_{\geq 0} \to [0, 1]$ is differentiable,
	non-increasing and satisfying $g(4) > 0$, \(
	g(z) \leq C_1\exp(-z^2/2)$ and $\sup_{z}|g'(z)| \leq C_2
    \) for some constants $C_1$ and $C_2$.
\end{assumption}
\begin{theorem}
    \label{thm.1}
    Under Assumptions \ref{assu.2} and \ref{assu.3}, for any $\delta \in (0, 1)$, with probability at least $1 - \delta$, 
    \[
        \mathcal{R}(K) \leq \widetilde{O}\bigg(\sqrt{H^3K|\mathcal{C}_{\sigma}|} + H^2|\mathcal{C}_{\sigma}||\widetilde{\mathcal{C}}_{\sigma}| + HK\sigma\bigg),
    \]
    where $|\mathcal{C}_{\sigma}|$ and $|\widetilde{\mathcal{C}}_{\sigma}|$ are the $\sigma$-covering numbers of $\mathcal{S} \times \mathcal{A}$ and $\mathcal{S}$, under $\rho$ and $\rho_{\mathcal{S}}$, respectively.
\end{theorem}
\begin{proof}
    See Appendix \ref{pot}.
\end{proof}

\cite{Domingues2020KernelBasedRL} leveraged the Kernel-UCBVI algorithm to obtain the first and state-of-the-art regret (see Theorem 1 therein) for  model-based finite horizon inhomogeneous continuous RL with smoothing kernels. Their regret is bounded by
\beqlb\label{Doming1}
\widetilde{O}\bigg(\sqrt{H^4K|\mathcal{C}_{\sigma}|} + H^3|\mathcal{C}_{\sigma}||\widetilde{\mathcal{C}}_{\sigma}| + HK\sigma\bigg).
\eeqlb
There are two differences between \citet{Domingues2020KernelBasedRL} and the current paper: (1) While $r_h\in[0,1]$ in \citet{Domingues2020KernelBasedRL}, we require a slightly weaker assumption $\sum_{h=1}^Hr_h\leq H$. (2). The Bernstein-type bonus  \eqref{est.5} may  improve the Hoeffding-type bonus in \citet{Domingues2020KernelBasedRL} by leveraging the variance of the next state. 
While it is not quite clear up to now but we believe that the reduction of the powers of $H$ from $\sqrt{H^4K|\mathcal{C}_{\sigma}|}$ and $H^3|\mathcal{C}_{\sigma}||\widetilde{\mathcal{C}}_{\sigma}|$ to  $\sqrt{H^3K|\mathcal{C}_{\sigma}|}$ and $H^2|\mathcal{C}_{\sigma}||\widetilde{\mathcal{C}}_{\sigma}|$ respectively stems from the synergy between Bernstein-style bonuses and kernel smoothing.

Theorem \ref{thm.1} provides a class of bounds that vary in $\sigma$. While we can minimize them to have the sharpest bound. It is not quite clear to interpret that sharpest bound. Presented in what follows is another form of the theory in terms of covering dimension, another number measuring richness of a metric space, so as to provide more interpretable regret bounds.
    The covering dimension of a metric space $\mathcal{G}$ associated with a  constant $c$ induced from its $\sigma$-covering numbers $|\mathcal{C}_{\sigma}^{\mathcal{G}}|$ is defined by
    \begin{equation}\label{dimension}
        d_{\mathcal{G}} = \min\{d \in \mathbb{N}^+ : |\mathcal{C}_{\sigma}^{\mathcal{G}}| \leq c\sigma^{-d} \hbox{ for all } \sigma > 0\}.
    \end{equation}
    See e.g., \citet{Sinclair2019AdaptiveDF} and \citet{Domingues2020KernelBasedRL}.
For instance, the covering dimension of a ball in $\mathbb{R}^d$ is $O(d)$, as its $\sigma$-covering number is $O(\sigma^{-d})$.
Because a covering dimension depends on its parameter $c$ only through a constant factor in the regret bound, we therefore drop $c$ in the notation $d_{\mathcal{G}}$.

Let $d_{_{SA}}$ and $d_{_S}$ be the covering dimensions of $\mathcal{S} \times \mathcal{A}$ and $\mathcal{S}$, respectively, and, for any $H$, $K$ and an $\alpha\in(0, 1)$ small enough such that $H^\alpha K^{-\frac{1}{d_{_{SA}}+d_{_S}+1}}\leq 1$,  denote by
\begin{equation}\label{tilde_alpha}\tilde{\alpha}=\min\{\frac{1+\alpha d_{_{SA}}}{2}, 1-\alpha, (d_{_{SA}}+d_{_S})\alpha\}.
\end{equation}
Then we have the following corollary.
\begin{corollary}
    \label{coro.1}
    With probability at least $1 - \delta$,
    $ \mathcal{R}(K) \leq \widetilde{O}\Big(H^{2-\tilde{\alpha}}K^{\frac{d_{_{SA}} + d_{_{S}}}{d_{_{SA}} + d_{_S} + 1}}\Big)
    $
and hence
   \beqlb\label{coro.1-1}
   \mathcal{R}(K) \leq \widetilde{O}\Big(H^{1+\frac{1}{d_{_{SA}}+2}}K^{\frac{d_{_{SA}} + d_{_{S}}}{d_{_{SA}} + d_{_S} + 1}}\Big).
    \eeqlb
\end{corollary}
\begin{proof}
    Take $\sigma = H^{\alpha}K^{-\frac{1}{d_{_{SA}}+d_{_S}+1}}<1$. Using \eqref{dimension} and \eqref{tilde_alpha}, we have
    \beqnn
    \sqrt{H^3K|\mathcal{C}_{\sigma}|} &\leq& H^{2-\frac{1+\alpha d_{_{SA}}}{2}}K^{\frac{d_{_{SA}} + d_{_{S}}}{d_{_{SA}} + d_{\mathcal{S} } + 1}},\\
  H^2|\mathcal{C}_{\sigma}||\widetilde{\mathcal{C}}_{\sigma}| &\leq&  H^{2-(d_{_{SA}}+d_{_{S}})\alpha}K^{\frac{d_{_{SA}} + d_{_{S}}}{d_{_{SA}} + d_{\mathcal{S} } + 1}},\\
  HK\sigma &\leq& H^{1+\alpha}K^{\frac{d_{_{SA}} + d_{_{S}}}{d_{_{SA}} + d_{\mathcal{S} } + 1}}.
  \eeqnn
The first part follows by substituting them into Theorem \ref{thm.1}.

 Inequality \eqref{coro.1-1} is then examined through the following two cases. On the one hand, if $K\geq H^{\frac{d_{_{SA}}+d_{_S}+1}{d_{_{SA}}+2}}$, then one can take $\alpha=\frac{1}{d_{_{SA}}+2}$, so that $\tilde{\alpha}=1-\frac{1}{d_{_{SA}}+2}$ by using the fact $d_{_{SA}}\geq d_{_{S}}\geq 1$. Therefore, the first inequality just proved gives rise to
   $$
   \mathcal{R}(K) \leq \tilde{O}\Big(H^{1+\frac{1}{d_{_{SA}}+2}}K^{\frac{d_{_{SA}} + d_{_{S}}}{d_{_{SA}} + d_{_S} + 1}}\Big).
   $$
On the other hand, if $K<H^{\frac{d_{_{SA}}+d_{_S}+1}{d_{_{SA}}+2}}$, then
 $$H^{1+\frac{1}{d_{_{SA}}+2}}K^{\frac{d_{_{SA}} + d_{_{S}}}{d_{_{SA}} + d_{_S} + 1}}\geq HK\geq \mathcal{R}(K).$$
Summing up, the second inequality follows.
\end{proof}


We have the following remarks on this bound.

Firstly, with $\sigma = K^{-\frac{1}{d_{SA}+d_{S}+1}}$, \citet{Domingues2020KernelBasedRL} obtained a regret of order $\mathcal{O}\!\left(H^{3} K^{\frac{2d_{SA}}{2d_{SA}+1}}\right)$ with probability at least $1-\delta$.
In contrast, the regret bound in inequality~(8) improves the dependence on the horizon $H$ by
$3 - \frac{d_{SA}+3}{d_{SA}+2} = 1 + \frac{d_{SA}+1}{d_{SA}+2}$, in which
the reduction of $1$ is attributed to the use of a Bernstein-type exploration bonus in the proposed KBVI-BUCB algorithm, while the remaining $\frac{d_{SA}+1}{d_{SA}+2}$ comes from a more refined choice of the kernel bandwidth, which is allowed to depend on both $H$ and $K$, rather than only on $K$ as in \citet{Domingues2020KernelBasedRL}.

Secondly, for model based RL with continuous state-action spaces the state-of-the-art upper bound of regret in terms of $H$ was derived in \citet{Sinclair2020AdaptiveDF} via adaptive discretization, as $$\mathcal{R}(K)\leq \begin{cases} \widetilde{O}\big(H^\frac{d_{_{SA}} + 2}{d_{_{SA}} + 1}K^{\frac{d_{_{SA}} + d_{\mathcal{S}} - 1}{d_{_{SA}} + d_{\mathcal{S}}}}\big), & d_{\mathcal{S}} > 2, \\
       \widetilde{O}\big(H^{\frac{d_{_{SA}} + 2}{d_{_{SA}} + 1}}K^{\frac{d_{_{SA}} + d_{\mathcal{S}} + 1}{d_{_{SA}} + d_{\mathcal{S}} + 2}}\big), &d_{\mathcal{S}} \leq 2.\end{cases}$$
Compared with this, the one in (\ref{coro.1-1}) improves by a reduction $\frac{d_{_{SA}} + 2}{d_{_{SA}} + 1}-\frac{d_{_{SA}} + 3}{d_{_{SA}} + 2}$ of the power of $H$ and, for the case of $d_{_S}\leq 2$, $\frac{d_{_{SA}} +d_{_S}+1}{d_{_{SA}} +d_{_S}+2}-\frac{d_{_{SA}} +d_{_S}}{d_{_{SA}} +d_{_S}+1}$ of the power of $K$.

In addition, in terms of $K$, we observe that when $d_{_S}=1$,  the  bound in (\ref{coro.1-1})  achieves the lower bound in continuous MDPs under Assumption \ref{assu.2}, i.e. $\Omega\left(K^{\frac{d_{_{SA}} + 1}{d_{_{SA}} + 2}}\right)$ proposed by \cite{JMLR:v15:slivkins14a}. However,  fundamental gaps persist for more general cases.

Finally, if $H^{\frac{d_{_{SA}}+d_{_S}+1 }{d_{_{SA}}+2}}>K$, then $\alpha\in [0, \frac{1}{d_{_{SA}}+2})$  so that  the corresponding $\tilde{\alpha}\leq \frac{1+\alpha d_{_{SA}}}{2}$. Consequently,
\beqnn
&&H^{2-\tilde{\alpha}}K^{\frac{d_{_{SA}} + d_{_{S}}}{d_{_{SA}} + d_{_S} + 1}}\geq HK\Big(H^{\frac{1-\alpha d_{_{SA}}}{2}}K^{-\frac{1}{d_{_{SA}}+d_{_S}+1}}\Big)
\\&&\quad \geq HK\times H^{\frac{1}{2}-\frac{d_{_{SA}}}{2(d_{_{SA}}+2)}}K^{-\frac{1}{d_{_{SA}}+d_{_S}+1}}\geq HK,
\eeqnn
which implies that the upper bound in Corollary \ref{coro.1} is larger than the natural bound $HK$. It thus raises an open  problem whether model-based methods with smoothing kernels can achieve a sublinear regret upper bound for finite-horizon online RL in this setting, even when the number of episodes $K$ is sufficiently large.

\section{Technique overview}\label{tech_over}
This section provides a proof sketch of Theorem \ref{thm.1}. All details are deferred to the Appendix. The analysis is divided into four parts:
\begin{itemize}
	\item[(i)]  Establish a new Bernstein concentration inequality for martingales that is critical to probability bounds.
    \item[(ii)]  Define a good event of high probability.
    \item[(iii)]  Show that $V_h^k(s)$ is an upper bound of $V_h^*(s)$ for all $(h, k, s) \in [H] \times [K] \times \mathcal{S}$ on the good event.
    \item[(iv)] Prove Theorem \ref{thm.1} by a decomposition in Eq \eqref{regret}.
\end{itemize}
\subsection{A New Bernstein concentration inequality}
Bernstein concentration inequality for martingales is widely applied to bound sums of random variables with a given probability and is widely used in RL, see
for example Lemma 11 in \cite{Zhang2021modelfree} and Lemma 3 in \cite{Domingues2020KernelBasedRL}. The following one is quite important in our theoretical analysis.
\begin{theorem}[Bernstein concentration inequality]
    \label{aux.4}
    Suppose $X_1, X_2, \cdots$ is a martingale difference sequence with $|X_i| \leq c \in \mathbb{R}_+$ almost surely. Then for any $\delta\in(0, e^{-1})$ and $n \in \mathbb{N}_+$,  with probability at least $1 - 2(\lfloor \log_2n\rfloor + 3)\delta$,
    \[
       \bigg|\sum_{i = 1}^nX_i\bigg| \leq 2\sqrt{\sum_{i = 1}^nX_i^2\log(2/\delta)}+7\max\{c, 1\}\log(2/\delta).
    \]
\end{theorem}
The proof is in Appendix \ref{proof_lem_2}.

\begin{remark} The right hand side of the inequality is a function of $X_i$ but not some deterministic quantities. We thus call this inequality an empirical Bernstein inequality. \cite{peel2013empirical} established another empirical Bernstein bound
 in an online-learning setting under the condition, for any $n\geq 1$,  $(X_{n+1},X_{n+2})$
 are conditionally independent given $(X_1,\cdots,X_n)$. Clearly, this assumption does not hold for general bounded martingale difference sequences.
 Our result closes this gap by providing an empirical Bernstein inequality that holds under the general martingale difference condition, without requiring any additional independence assumptions.
\end{remark}

\subsection{A Good Event}

We define the following event:
\beqnn
\Bigg\{\big|\big(\widehat{P}^k_h - P_h\big)V_{h+1}^*(s, a)\big|
&&\lesssim \sqrt{\frac{\widehat{\V}^k(V_{h+1}^*, s, a)}{C^k_h(s, a)}}+ \frac{H}{C^k_h(s, a)}+ \sigma,
\\ &&\qquad\qquad\qquad\forall\ (s, a, h, k) \in \mathcal{S} \times \mathcal{A} \times [H] \times [K]\Bigg\}.
\eeqnn
This event ensures that the estimation error $|(\widehat{P}_h^k - P_h)V_{h+1}^*(s,a)|$ is uniformly controlled by the exploration bonus $b_h^k(s,a)$ across all $(s,a,h,k)$. We show that this event holds with high probability; see Appendix~\ref{app.goe} for a detailed proof.

\subsection{Optimism}

The second step is to show that $V^k_h$ serves as an upper bound on $V^*_h$ under the good event.

We prove this by backward induction on $h$. At the base case, $V^*_{H+1}(s) = V^k_{H+1}(s) = 0$ by definition. Suppose that $V^k_{h+1}(s) \geq V^*_{h+1}(s)$ holds for all $s \in \mathcal{S}$. Then, on the good event, the exploration bonus ensures that
\begin{align*}
    \widetilde{Q}_h^k(s,a) &= r_h(s,a) + \widehat{P}_h^k V_{h+1}^k(s,a) + b_h^k(s,a) \\
    &\geq r_h(s,a) + P_h V_{h+1}^*(s,a) \\
    &= Q_h^*(s,a),
\end{align*}
for all $(s, a) \in \mathcal{S} \times \mathcal{A}$. Since $Q_h^*$ is $L_h$-Lipschitz, it follows that for any $(s', a') \in \mathcal{S} \times \mathcal{A}$,
\beqnn
    Q_{h}^*(s, a) &\leq& Q_{h}^*(s', a') + L_h\rho\big[(s, a), (s', a')\big]
    \\&\leq& \widetilde{Q}^k_{h}(s', a') + L_h\rho\big[(s, a), (s', a')\big].
\eeqnn
By the definition of $Q_{h}^k$, taking the infimum over $(s', a')$ yields $Q_h^*(s, a) \leq Q^k_{h}(s, a)$ for all $(s, a) \in \mathcal{S} \times \mathcal{A}$, which in turn implies $V^k_{h}(s) \geq V^*_{h}(s)$ for all $s \in \mathcal{S}$. See Appendix~\ref{app.opt} for the complete proof.

\subsection{Regret Analysis}

Optimism yields the following regret decomposition. Since $V_1^*(s_1^k) \leq V_1^k(s_1^k)$ on the good event,
\begin{align*}
    \mathcal{R}(K) &= \sum_{k=1}^K \bigl(V_1^*(s_1^k) - V_1^{\pi^k}(s_1^k)\bigr) \\
    &\leq \sum_{k=1}^K \bigl(V_1^k(s_1^k) - V_1^{\pi^k}(s_1^k)\bigr).
\end{align*}
Noting that $V^k_{H+1}(s_{H+1}^k) \equiv 0$ for all $k$, we may telescope:
$$
    V_1^k(s_1^k) = \sum_{h=1}^H \bigl(V_h^k(s_h^k) - V_{h+1}^k(s_{h+1}^k)\bigr).
$$
We decompose each summand as
\begin{align*}
    &V_h^k(s_h^k) - V_{h+1}^k(s_{h+1}^k)
    = \bigl(P_h V_{h+1}^k(s_h^k,a_h^k) - V_{h+1}^k(s_{h+1}^k)\bigr) \\
    &\qquad\qquad\qquad\qquad\qquad\qquad + \bigl(V_h^k(s_h^k) - r_h(s_h^k,a_h^k) - P_h V_{h+1}^k(s_h^k,a_h^k)\bigr) + r_h(s_h^k,a_h^k),
\end{align*}
and write
$$
    V_1^{\pi^k}(s_1^k) = \sum_{h=1}^H r_h(s_h^k,a_h^k) - \biggl(\sum_{h=1}^H r_h(s_h^k,a_h^k) - V_1^{\pi^k}(s_1^k)\biggr).
$$
Defining
\beqnn
    I_1 &=& \sum_{k=1}^K \biggl(\sum_{h=1}^H r_h(s_h^k, a_h^k) - V^{\pi^k}_1(s_1^k)\biggr), \\
    I_2 &=& \sum_{k=1}^K \sum_{h=1}^H \bigl(P_h V^k_{h+1}(s_h^k, a_h^k) - V_{h+1}^k(s^k_{h+1})\bigr), \\
    I_3 &=& \sum_{k=1}^K \sum_{h=1}^H \bigl(V_h^k(s_h^k) - r_h(s_h^k, a_h^k) - P_h V^k_{h+1}(s_h^k, a_h^k)\bigr),
\eeqnn
we obtain
\begin{align}
    \label{regret}
    \mathcal{R}(K) \leq I_1 + I_2 + I_3.
\end{align}

\noindent We proceed to bound each term separately. Lemma~\ref{reb.1} and Lemma~\ref{reb.2} establish the following bounds:
\begin{align}
    I_1 &\leq \sqrt{2H^2 K\eta}, \label{eq:I_1} \\
    I_2 &\leq \widetilde{O}\!\left(\sqrt{H^3 K|\mathcal{C}_{\sigma}|} + \sqrt{H^2 |\mathcal{C}_{\sigma}||\widetilde{\mathcal{C}}_{\sigma}|} + \sqrt{HK\sigma}\right), \label{eq:I_2} \\
    I_3 &\leq \widetilde{O}\!\left(\sqrt{H^3 K|\mathcal{C}_{\sigma}|} + H^2 |\mathcal{C}_{\sigma}||\widetilde{\mathcal{C}}_{\sigma}| + HK\sigma\right). \label{eq:I_3}
\end{align}
Combining \eqref{regret}--\eqref{eq:I_3}, we conclude that, with high probability,
\[
    \mathcal{R}(K) \leq \widetilde{O}\!\left(\sqrt{H^3 K|\mathcal{C}_{\sigma}|} + H^2 |\mathcal{C}_{\sigma}||\widetilde{\mathcal{C}}_{\sigma}| + HK\sigma\right).
\]
See Appendix~\ref{pot} for the complete proof.

\section{Experiments}
\label{expe}

This section presents an experiment comparing the performance of KBVI-BUCB against the following baselines. Kernel-UCBVI \citep{Domingues2020KernelBasedRL} is the most closely related method, which also combines kernel smoothing with value iteration but relies on Hoeffding's inequality to construct its exploration bonus. Kernel-VI is an ablation that removes the exploration bonus entirely, serving as a reference point for assessing the value of principled exploration. AdaMB \citep{Sinclair2020AdaptiveDF} is a model-based method that adaptively discretizes the state-action space, while Adaptive Q-learning \citep{Sinclair2019AdaptiveDF} and Net-Based Q-learning \citep{Song2019EfficientMR} are representative model-free approaches based on adaptive partitioning and
$\varepsilon$-nets. In this paper, we consider the Puddle World environment.

The algorithms are evaluated in a continuous-state Puddle World environment,  with a state space $[-1, 1]^2$, as illustrated in Figure \ref{fig2}.
\begin{figure}[H]
\centering
\includegraphics[width=0.55\textwidth]{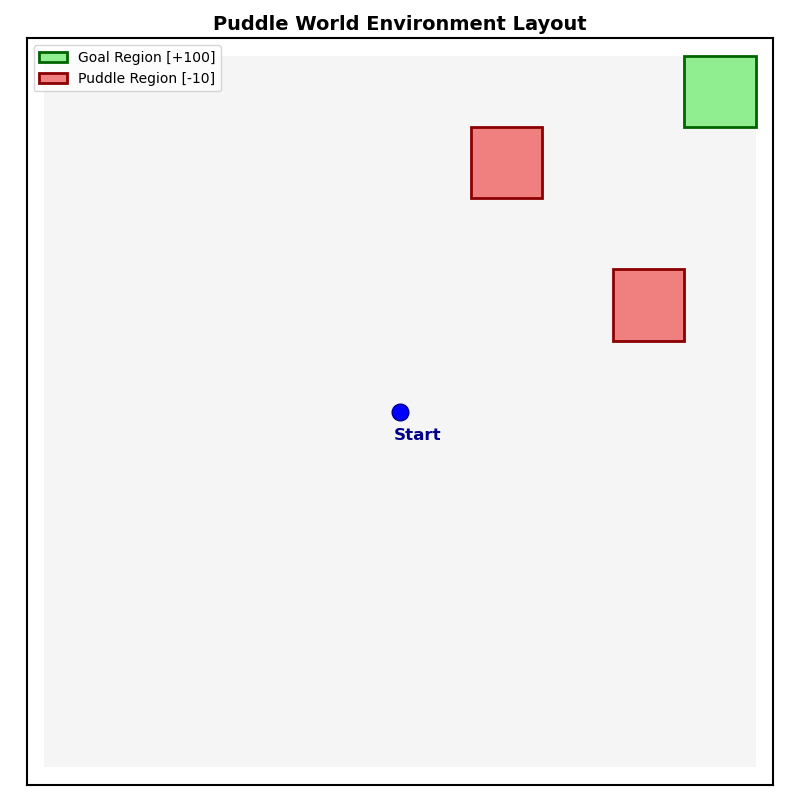}
\caption{Schematic of the Puddle World environment. The agent starts at the Start position and must navigate to the green goal region while avoiding the two red puddle areas.}
\label{fig2}
\end{figure}

The agent has four discrete actions (left, right, up, down), each inducing a deterministic displacement of 0.1, perturbed by Gaussian noise for stochasticity. The target region $[0.8, 1] \times [0.8, 1]$ yields a reward of 100, while entering either of the two puddle regions $[0.2, 0.4] \times [0.6, 0.8]$ and $[0.6, 0.8] \times [0.2, 0.4]$ results in a penalty of $-10$. All other states provide zero reward. The agent starts at $(0, 0)$. The kernel estimation is carried with the Euclidean distance and a Gaussian kernel with bandwidth $\sigma = 0.025$. Additionally, to reduce computational costs, we employ representative
states as in \cite{kveton2012kernel} in experiments, merging nearby states within a 0.02 distance
threshold to significantly accelerate the algorithm's runtime while preserving spatial resolution. The performance is the total reward obtained. Numerical results are shown in Figure \ref{fig1}.

\begin{figure}[h]
\centering
\includegraphics[width=0.8\textwidth]{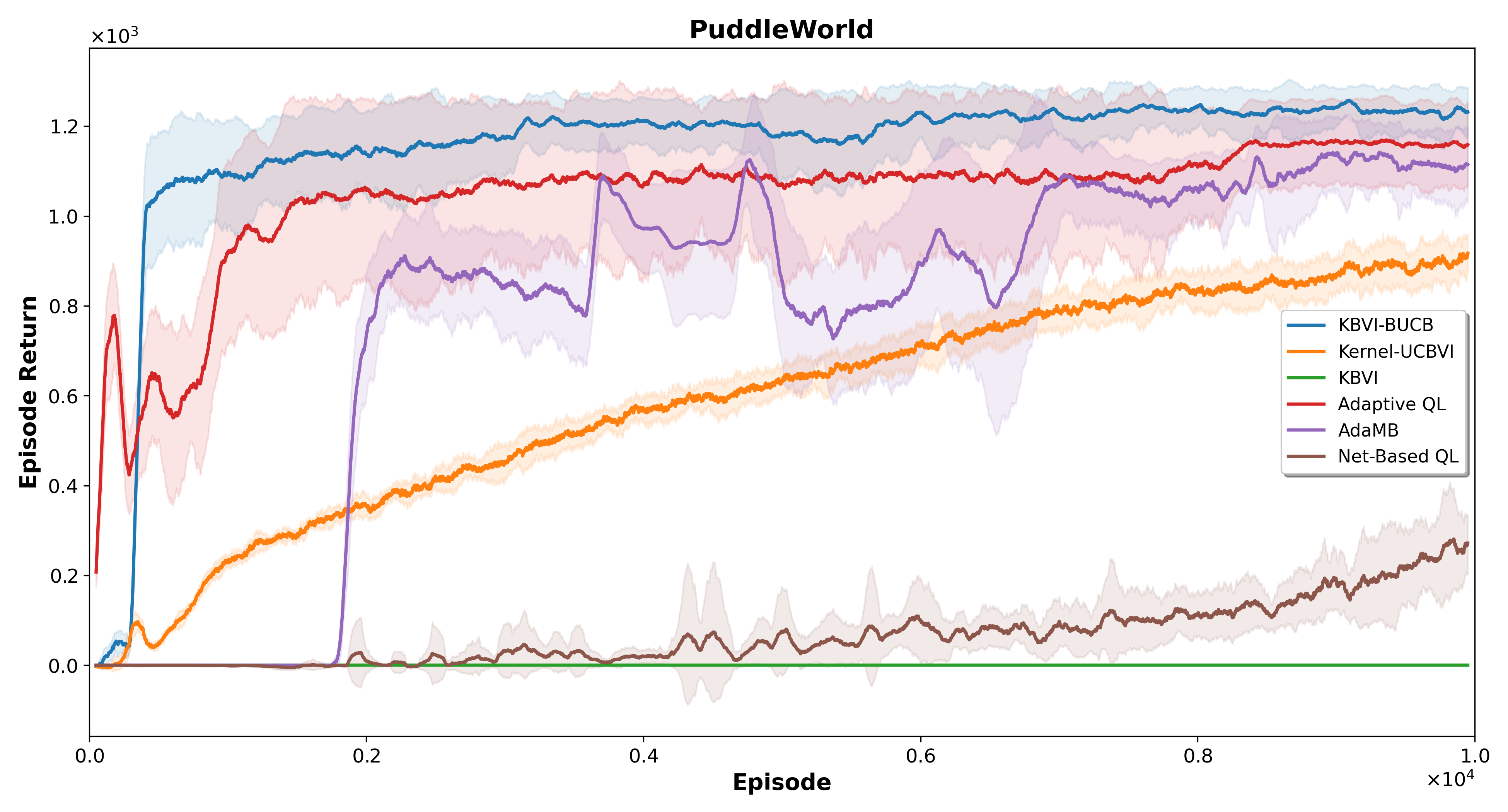}
\caption{Returns obtained by KBVI-BUCB and 5 baselines. We run experiments with 8 random seeds. Solid lines show the mean and shaded regions show one standard deviation across seeds.}
\label{fig1}
\end{figure}

As shown in Figure \ref{fig1}, KBVI-BUCB converges the fastest and achieves the highest return among all methods. Notably, all other methods that successfully learn, including Kernel-UCBVI, Adaptive QL, AdaMB, and Net-Based QL, rely on Hoeffding-type exploration bonuses, and KBVI-BUCB outperforms all of them, demonstrating the advantage of Bernstein-based bounds across different algorithmic frameworks. Among the baselines, Adaptive QL and AdaMB benefit from adaptive partitioning and perform competitively, while Net-Based QL, which uses a fixed discretization, learns more slowly. As noted in \cite{Domingues2020KernelBasedRL}, kernel-based methods may also benefit from adaptive bandwidth selection, which we leave as a promising direction for future work. KBVI stays near zero throughout training because the two puddle regions with negative rewards lie between the start point and the target. Without exploration bonuses, the values estimated by KBVI at the puddle regions are likely to be negative, which, together with the kernel method, prevents the agent from approximating the target area. Figure \ref{fig3} zooms into the first 100 episodes and shows that KBVI is not completely static. KBVI stays near zero throughout training because the two puddle regions with negative rewards lie between the start point and the target. Without exploration bonuses, the values estimated by KBVI at the puddle regions are likely to be negative, which, together with the kernel method, prevents the agent from approximating the target area.

\begin{figure}[H]
	\centering
	\includegraphics[width=0.8\textwidth]{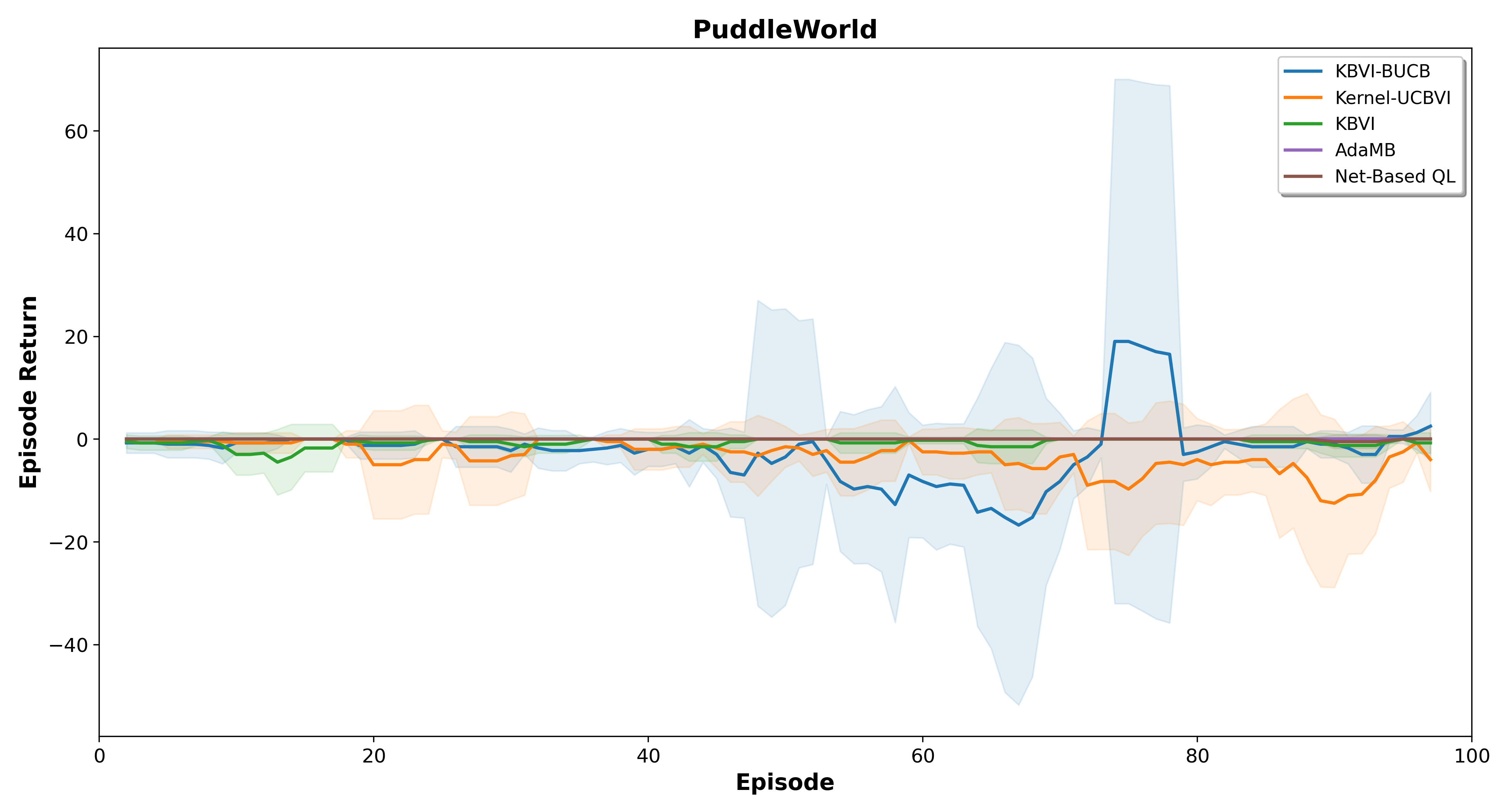}
	\caption{Returns on PuddleWorld for the first 100 episodes. Adaptive QL is omitted as its early returns are on a much larger scale.}
	\label{fig3}
\end{figure}

We also tested under different reward settings. The results, reported in Appendix \ref{sec:experimental}, show that KBVI-BUCB consistently achieves the highest return across all settings.
\section{Discussion}
\label{discussion}
In this paper, we proposed the  KBVI-BUCB algorithm for finite-horizon RL in metric state-action spaces. Our analysis reveals that the KBVI-BUCB achieves a high-probability regret bound
$\widetilde{O}\big(H^{2-\tilde{\alpha}}K^{\frac{d_{_{SA}} + d_{_{S}}}{d_{_{SA}} + d_{_{S}} + 1}}\big)$
in inhomogeneous MDPs. This represents an improvement in terms of $H$ over the $\widetilde{O}\big(H^3K^{\frac{2d_{_{SA}}}{2d_{_{SA}} + 1}}\big)$ bound of \cite{Domingues2020KernelBasedRL}. Furthermore, when $H^{\frac{d_{_{SA}}+d_{_S}+1}{d_{_{SA}}+2}}\leq K$ which is implied by $H^2<K$, the upper bound can be written explicitly as
 $\widetilde{O}\big(H^{\frac{d_{_{SA}}+3}{d_{_{SA}}+2}}K^{\frac{d_{_{SA}} + d_{_{S}}}{d_{_{SA}} + d_{_{S}} + 1}}\big)$
which improves the state-of-the-art upper bound of regret in terms of $H$ \cite{Sinclair2020AdaptiveDF}.

While our work advances the theoretical frontier of kernel-based RL, a fundamental gap persists between the current upper bound and the lower bound $\Omega\left(K^{\frac{d_{_{SA}} + 1}{d_{_{SA}} + 2}}\right)$ in terms of $K$. As discussed in \cite{zhang2024settling}, bridging this gap possibly requires further efforts to decouple the statistical dependence between estimated transition models $\widehat{P}_h^k$ and subsequent samples.

In addition, when $H^{\frac{d_{_{SA}}+d_{_S}+1}{d_{_{SA}}+2}}>K$, the upper bound of the regret is essentially linear in $HK$. It deserves further study whether one can obtain a sublinear upper bound or not using the
model-based method with smoothing kernels for the finite-horizon online RL in this situation.

\backmatter

%
%
%
%
%
%
%


\begin{appendices}
\renewcommand{\thetable}{\arabic{table}}
\setcounter{table}{1}
\section{Notations and auxiliary results}\label{appendix1}
We begin by introducing some notations used in Appendix.
\subsection{Notations}

\begin{table}[h]
\caption{Table of notations}\label{notations}
    \begin{tabular}{@{}ll@{}}
    \toprule[2pt]
    Notation       & Explanation   \\
    \midrule[2pt]
    Lip($f$) & Lipschitz constant of function $f$\\
    $\rho$  & Metric on state-action space \\
    $\lambda_p, \lambda_r$ & Lipschitz constants of rewards and transition distributions (Assumption \ref{assu.2})\\
    $g(\cdot)$      & A differentiable kernel function\\
    $|\mathcal{C}_{\sigma}|$ &  $\sigma$-covering number of metric space $(\mathcal{S} \times \mathcal{A}, \rho)$\\
    $|\widetilde{\mathcal{C}}_{\sigma}|$ &  $\sigma$-covering number of metric space $(\mathcal{S}, \rho_{\mathcal{S}})$\\
    $C_1, C_2$  &   Two constants postulated in Assumption \ref{assu.3}\\
    $C_3$ & $\frac{1}{\beta}+\frac{1}{g(4)}\log(1+\frac{g(4)K}{\beta})$\\
    $L_1$ & $\sum_{h= 1}^{H}\lambda_{r}\lambda_{p}^{H - h}$\\
    $\gamma$  & $4\sqrt{\log(HK/\beta + e)}$.\\
    $\eta$ & $\log(2/\delta)$\\
    $\eta_1$ & $C_1\eta$ \\
    $\eta_2$ & $7\max\{C_1, 1\}\eta + \beta$ \\
    $\eta_3$ & $8\lambda_pL_1\eta\gamma$ \\
    $\eta_4$ & $16\eta_1 + 2\eta_2$ \\
    $\eta_5 $ & $ (2C_1 + 4)\eta + \beta$\\
    $D_{\delta}$ & $16\sqrt{\frac{C_2\eta_1\eta}{\beta^3}} + \big(C_2\eta_4 + 2C_2\big)\frac{\sigma}{H\beta^2} + \lambda_pL_1\frac{\sigma^2}{KH^2} + \eta_3$\\
    $D_{L,F,\sigma}$ &  $\eta_5 + |\widetilde{\mathcal{C}}_{\sigma}|\log\frac{2F}{L\sigma}$\\
    $D_{L, F}$ & $ 2\sqrt{\frac{4C_1\lambda_pL\sigma}{\beta KF^2} + \frac{C_1C_2}{\beta^2}} + (\eta_5F + 3)\frac{C_2\sigma}{F^2\beta^2} + \lambda_pL\frac{\sigma^2}{KF^2} + 8\lambda_pL\gamma$\\
    $\widetilde{D}_{L, F}$ & $ D_{L, F} + 12L$\\
    $F_1(\eta, H, \sigma)$ & $12D_{L_1,H,\sigma} + 165\eta_4$ \\
\bottomrule[2pt]
\end{tabular}
\end{table}
Below, we define some important concepts. For $(k, h)\in[K]\times[H]$,
    define $$\mathcal{F}_{h}^k = \sigma\Big(\{s_{h}^{l}, a_{h}^{l}, s_{h + 1}^{l}\}_{l < k} \cup \{s^{k}_{h}, a^{k}_{h}\}\Big),$$ and let $\{\mathcal{F}_{h}^k\}_{k}$ be the corresponding filtration.  The weights are defined by :
\[
    \begin{aligned}
        \mathnormal{w}_h^l(s, a) &= g({\rho((s, a), (s_h^l, a_h^l))\over\sigma});\\
	\widetilde{\mathnormal{w}}_h^{k, l}(s, a) &=        \frac{w_h^l(s, a)}{\beta + \sum_{n = 1}^{k-1}w_h^n(s, a)},
    \end{aligned}
\]
where $0 < \beta \leq 1$ is a regulation. In addition, the generalized counts of $(s, a)$ at the beginning of $k$-th episode are given by $$C^k_h(s, a) = \beta + \sum_{l = 1}^{k-1}w^{l}_h(s, a).$$ We define the estimators of transition distribution at the beginning of the $k$-th episode as
\[
    \widehat{P}^k_h(\mathcal{O}|s, a) = \sum_{l = 1}^{k-1}\widetilde{\mathnormal{w}}^{k, l}_h(s, a)\mathbf{1}_{\mathcal{O}}(s_{h+1}^l),
\]
for any $\mathcal{O} \subseteq \mathcal{S}$. So, the expectations of function $V : \mathcal{S} \to \mathbb{R}$ under $P_h(\cdot|s, a)$ and $\widehat{P}_h^k(\cdot|s, a)$ are
\[
    \begin{aligned}
        &P_hV(s, a) := \E_{s' \sim P_h(\cdot|s, a)}[V(s')];\\
        &\widehat{P}^k_hV(s, a) := \E_{s' \sim \widehat{P}_h(\cdot|s, a)}[V(s')] = \sum_{l = 1}^{k-1}\widetilde{\mathnormal{w}}_h^{k, l}(s, a)V(s_{h+1}^l).
    \end{aligned}
\]
In addition, the variances of $V$ under $P_h(\cdot|s, a)$ and $\widehat{P}^k_h(\cdot|s, a)$ are given by
\[
    \begin{aligned}
        &\V_h(V , s, a) := {P}_h\big[V - {P}_hV(s, a)\big]^2(s, a);\\
        &\widehat{\V}^k_h(V, s, a) := \widehat{P}^k_h\big[V - \widehat{P}^k_hV(s, a)\big]^2(s, a) = \sum_{l = 1}^{k-1}\widetilde{\mathnormal{w}}_h^{k, l}(s, a)\big[V(s_{h+1}^l) - \widehat{P}^k_hV(s, a)\big]^2.
    \end{aligned}
\]
We also define bonus $b_h^k(s, a)$ as follows:
\[
    b_{h}^k(s, a) = 9\sqrt{\frac{\widehat{\V}^k_h(V_{h + 1}^k, s, a)\eta_4}{C^k_h(s, a)}} + \frac{162H\eta_4}{C^k_h(s, a)} + \sigma D_{\delta},
\]
where
\begin{align*}
    D_{\delta} = 16\sqrt{\frac{C_2\eta_1\eta}{\beta^3}} + \big(C_2\eta_4 + 2C_2\big)\frac{\sigma}{H\beta^2} + \lambda_pL_1\frac{\sigma^2}{KH^2} + 32\lambda_pL_1\eta\sqrt{\log(HK/\beta + e)}.
\end{align*}
According to the algorithm, 
for each $(s, a, h, k) \in \mathcal{S} \times \mathcal{A} \times [H] \times [K]$
\[
    \widetilde{Q}_h^k(s, a) = r_h(s, a) + \widehat{P}^k_hV_{h+1}^k(s, a) + b_h^k(s, a)
\]
and
\[
     Q_h^k(s, a) = \min_{l \in [k-1]}\big\{\widetilde{Q}_h^k(s_h^l, a_h^l) + L_h\rho[(s, a), (s_h^l, a_h^l)]\big\}.
\]
The corresponding $V$ function is defined by
\[
    V_{h}^k(s) = \min\{\max_{a}Q_{h}^k(s, a), H\}.
\]
Our goal is to minimize the regret :
\[
    \mathcal{R}(K) = \sum_{k = 1}^K\Big(V^*_1(s_1^k) - V^{\pi^k}_1(s_1^k)\Big).
\]

\section{Proof of Theorem \ref{aux.4}}\label{proof_lem_2}
 We first establish a lemma below that takes an important role in proving Theorem \ref{aux.4}.  It is similar to Lemma 38 in \citet{boone2024achieving} that is attributed to \citet{Audibert2009} who analyzes, however, i.i.d. rather than  martingale difference sequences, so that  here appears the first  rigorous proof of this type of Berstein's inequality for martingales. 
\begin{lemma}[Bernstein's inequality for martingales]
    \label{aux.2}
    Suppose $X_1, X_2, \cdots$ is a martingale difference sequence where $|X_i| \leq c \in \mathbb{R}_+$ almost surely. Then for any $n \in \mathbb{N}_+$, we have with probability at least $1 - (\lfloor\log_2n\rfloor + 3)\delta$,
    \[
        \bigg|\sum_{i = 1}^nX_i\bigg| \leq 2\sqrt{\sum_{i = 1}^n\E[X_i^2|X_{<i}]\log(2/\delta)} +\frac{2c}{3} \log(2/\delta) + c\sqrt{\log(2/\delta)}.
    \]
\end{lemma}
\begin{proof} Without loss of generality, we suppose $c = 1$. For $n\geq 1$, let $V_n=\sum_{i = 1}^n\E[X_i^2|X_{<i}]$, Freedman's inequality \cite{freedman1975tail} implies that for any $x, y>0$,
\beqnn
\pr\Big(\sum_{i=1}^n X_i\geq x, V_n\leq y\Big)\leq e^{-\frac{x^2}{2y + 2cx/3}},
\eeqnn
 which implies
    \[
        \pr\Big(\sum_{i=1}^n X_i \geq \sqrt{2y\log(2/\delta)} + \frac{2\log(2/\delta)}{3}, V_n \leq y\Big) \leq \delta/2.
    \]
 Therefore, for any $j$,
    \[
        \pr\Big(\sum_{i=1}^n X_i \geq \sqrt{2^{j + 1}\log(2/\delta)} + \frac{2\log(2/\delta)}{3}, V_n \leq 2^j\Big) \leq \delta/2.
    \]
   Hence,
     \beqlb\label{bernstein-1}
            &&\pr\bigg(\sum_{i=1}^n X_i \geq 2\sqrt{V_n\log(2/\delta)} + \sqrt{2\log(2/\delta)} + \frac{2\log(2/\delta)}{3}\bigg)\nonumber
            \\ &&\leq\sum_{j = 0}^{\lfloor\log_2n\rfloor+1}\pr\bigg(\sum_{i=1}^n X_i  \geq 2\sqrt{V_n\log(2/\delta)} + \sqrt{2\log(2/\delta)} + \frac{2\log(2/\delta)}{3}, 2^{j - 1} < V_n \leq 2^j\bigg)\nonumber
            \\&&\quad+  \pr\bigg(\sum_{i=1}^n X_i \geq 2\sqrt{V_n\log(2/\delta)} + \sqrt{2\log(2/\delta)} + \frac{2\log(2/\delta)}{3}, 0 \leq V_n \leq 1/2\bigg)\nonumber
            \\&& \leq\sum_{j = 0}^{\lfloor\log_2n\rfloor+1}\pr\bigg(\sum_{i=1}^n X_i \geq \sqrt{2^{j + 1}\log(2/\delta)} + \sqrt{2\log(2/\delta)} + \frac{2\log(2/\delta)}{3}, V_n \leq 2^j\bigg)\nonumber
            \\&&\quad+  \pr\bigg(\sum_{i=1}^n X_i \geq \sqrt{2\log(2/\delta)} + \frac{2\log(2/\delta)}{3}, V_n \leq 1/2\bigg)\nonumber
            \\&&\leq (\lfloor\log_2n\rfloor + 3)\delta/2.
    \eeqlb
 The same arguments can be applied to $-X_i$, and the desired result follows immediately.
\end{proof}

The following Lemma generalizes Theorem 10 in
\citet{maurer2009empirical} for i.i.d. random variables, and characterizes confidence bounds of deviation of the sum of a sequence of
bounded random variables from its conditional expectation.
\begin{lemma}
    \label{aux.3}
    Let $Y_1, Y_2, \cdots$ be a sequence of random variables such that $0 \leq Y_{i} \leq c$.
    Then with probability at least $1 - 2(\lfloor\log_2n\rfloor + 3)\delta$ where $\delta\in(0,e^{-1})$,
    \[
            \sqrt{\sum_{i = 1}^{n}Y_{i}} \leq \sqrt{\sum_{i = 1}^{n}\E\Big[Y_i\Big|Y_{<i}\Big]} + 2\sqrt{c\log(1/\delta)} \;\;\hbox{and}\;\;
            \sqrt{\sum_{i = 1}^{n}\E\Big[Y_i\Big|Y_{<i}\Big]} \leq \sqrt{\sum_{i = 1}^{n}Y_{i}} + 4\sqrt{\frac{c\log(1/\delta)}{3}}.
    \]
\end{lemma}
\begin{proof}
    Without loss of generality, we suppose $c = 1$. Let $Y$ be a random variable such that $0 \leq Y \leq 1$, then for any $\lambda$,
    \[
        \E\Big[e^{\lambda Y}\Big] \leq \exp\Big\{\E\big[e^{\lambda Y} - 1\big]\Big\}.
    \]
    Note that $\E\Big[e^{\lambda Y}\Big] \leq e^{\E[Y](e^\lambda - 1)}$ because
    \(
        e^{Y\lambda} - 1 \leq Ye^\lambda - Y
    \)
    for any $0 \leq Y \leq 1$.
    Thus,
    \[
        \E\Big[e^{\lambda(Y - \E[Y])}\Big] \leq e^{\E[Y](e^{\lambda} - \lambda - 1)}.
    \]
    Now let $S_n = \sum_{i = 1}^{n}\big(Y_i - \E\big[Y_i\big|Y_{<i}\big]\big), \mu_i = \E[Y_{i}|Y_{<i}], \bar{\mu}_n = \sum_{i = 1}^{n}\mu_{i}$. Then
        \begin{align}
            \E\Big[e^{\lambda S_n}\Big|Y_{<n}\Big]
            &= e^{\lambda S_{n - 1}}\cdot\E\Big[e^{\lambda(Y_{n} - \E[Y_{n}|Y_{<n}])}\Big|Y_{<n}\Big]\notag\\
            &\leq e^{\lambda S_{n - 1}}\cdot e^{\mu_{n}(e^{\lambda} - \lambda - 1)},
        \end{align}
    which implies that $\{e^{\lambda S_n-{\bar\mu}_n(e^{\lambda} - \lambda - 1)}\}_n$ is a super-martingale. Therefore, for any $\lambda$
      \begin{align}
            \E\Big[e^{\lambda S_n-{\bar\mu}_n(e^{\lambda} - \lambda - 1)}\Big]\leq 1,\label{lem20-1}
        \end{align}
 and hence,
    \begin{align}
        \pr\Big(S_n \geq \varepsilon, \bar{\mu}_n\leq v \Big) &\leq \inf_{\lambda \geq 0}e^{-\lambda\varepsilon+v(e^{\lambda} - \lambda - 1)}\E[e^{\lambda S_n-{\bar\mu}_n(e^{\lambda} - \lambda - 1)}]\notag\\
        &\leq \inf_{\lambda \geq 0}e^{\nu(e^\lambda - \lambda - t\lambda - 1)}\notag\\
        &\leq e^{\nu(t - (1 + t)\log(1 + t))}\label{con.3.1},
    \end{align}
    where $\nu$ is a positive constant and $t = \varepsilon/\nu$. Note that
    \(
        (1 + t)\log(1 + t) - t \geq \frac{t^2}{2 + 2t/3}
    \)
    for all $t \geq 0$, equation (\ref{con.3.1}) yields
    \[
        \pr(S_n \geq \varepsilon, {\bar \mu}_n\leq \nu) \leq e^{-\frac{\varepsilon^2}{2\nu + 2\varepsilon/3}},
    \]
  By Lemma \ref{aux.2}, we can readily get that with probability at least $1 - (\lfloor\log_2n\rfloor + 3)\delta$,
    \[
        \begin{aligned}
            \sum_{i = 1}^{n}Y_{i} &\leq \sum_{i = 1}^{n}\E\Big[Y_i\Big|Y_{<i}\Big] + 2\sqrt{\sum_{i = 1}^{n}\E\Big[Y_i\Big|Y_{<i}\Big]\log(1/\delta)} + \sqrt{\log(1/\delta)} + \frac{2\log(1/\delta)}{3}\\
            &\leq \bigg(\sqrt{\sum_{i = 1}^{n}\E\Big[Y_i\Big|Y_{<i}\Big]} + \sqrt{\log(1/\delta)}\bigg)^2 + \frac{2\log(1/\delta)}{3},
        \end{aligned}
    \]
    Therefore,
    \begin{equation}
        \label{con.3.2}
        \sqrt{\sum_{i = 1}^{n}Y_{i}} \leq \sqrt{\sum_{i = 1}^{n}\E\Big[Y_i\Big|Y_{<i}\Big]} + 2\sqrt{\log(1/\delta)}.
    \end{equation}
    The first inequality is thus proved.

To prove the second inequality, we use (\ref{lem20-1}) with $\lambda \leq 0$. Since $S_n=\sum_{i=1}^nY_i-\bar{\mu}_n\geq -{\bar\mu}_n$, we have that for any $0 < \varepsilon$,
    \[
        \begin{aligned}
            \pr\Big(S_n \leq -\varepsilon, {\bar\mu}_n \leq \nu\Big)&=\Big(S_n \leq -(\varepsilon\wedge\nu), {\bar\mu}_n \leq \nu\Big)
            \\&\leq \inf_{\lambda \leq 0}e^{\nu(e^\lambda - \lambda + t\lambda - 1)}\leq e^{-\nu(t + (1 - t)\log(1 - t))},
        \end{aligned}
    \]
    where $t = (\varepsilon\wedge \nu)/\nu$. Use the fact $t + (1 - t)\log(1 - t) \geq \frac{t^2}{2 - 2t/3}$ for all $0 \leq t \leq 1$ and we have
    \[
        \pr\Big(S_n \leq -\varepsilon, \mu \leq \nu\Big) \leq e^{-\frac{(\varepsilon\wedge\nu)^2}{2\nu - 2(\varepsilon\wedge\nu)/3}}.
    \]
    Similarly, considering $\nu = 2^j, j = -1, 0, 1, \cdots$, we have
    \[
        \pr\bigg(S_n \leq -\sqrt{2^{j + 1}\log(1/\delta)} + \frac{2\log(1/\delta)}{3}, \mu \leq 2^j\bigg) \leq \delta.
    \]
    Hence,
    \[
        \begin{aligned}
            &\pr\bigg(S_n \leq -2\sqrt{\mu\log(1/\delta)} - \sqrt{\log(1/\delta)} + \frac{2\log(1/\delta)}{3}\bigg)\\
            \leq &\sum_{j = 0}^{\lfloor\log_2n\rfloor+1}\pr\bigg(S_n \leq -2\sqrt{\mu\log(1/\delta)} - \sqrt{\log(1/\delta)} + \frac{2\log(1/\delta)}{3}, 2^j < \mu \leq 2^{j + 1}\bigg)\\
            & + \pr\bigg(S_n \leq -2\sqrt{\mu\log(1/\delta)} - \sqrt{\log(1/\delta)} + \frac{2\log(1/\delta)}{3}, 0 \leq \mu \leq 1/2\bigg)\\
            \leq &\sum_{j = 0}^{\lfloor\log_2n\rfloor+1}\pr\bigg(S_n \leq -\sqrt{2^{j + 1}\log(1/\delta)} - \sqrt{2\log(1/\delta)} + \frac{\log(1/\delta)}{3}, \mu \leq 2^{j + 1}\bigg)\\
            & + \pr\bigg(S_n \leq - \sqrt{\log(1/\delta)} + \frac{2\log(1/\delta)}{3}, \mu \leq 1/2\bigg)\\
            \leq &(\lfloor\log_2n\rfloor + 3)\delta.
        \end{aligned}
    \]
    It follows that with probability at least $1 - (\lfloor\log_2n\rfloor + 2)\delta$,
    \[
        \begin{aligned}
            \sum_{i = 1}^{n}Y_{i} &\geq \sum_{i = 1}^{n}\E\Big[Y_i\Big|Y_{<i}\Big] - 2\sqrt{\sum_{i = 1}^{n}\E\Big[Y_i\Big|Y_{<i}\Big]\log(1/\delta)} - \sqrt{\log(1/\delta)} + \frac{2\log(1/\delta)}{3}\\
            &\geq \bigg(\sqrt{\sum_{i = 1}^{n}\E\Big[Y_i\Big|Y_{<i}\Big]} - \sqrt{\log(1/\delta)}\bigg)^2 - \frac{4\log(1/\delta)}{3}.
        \end{aligned}
    \]
    Therefore, with probability at least $1 - (\lfloor \log_2n\rfloor + 3)\delta$,
    \begin{equation}
        \label{con.3.3}
        \sqrt{\sum_{i = 1}^{n}Y_{i}} \geq \sqrt{\sum_{i = 1}^{n}\E\Big[Y_i\Big|Y_{<i}\Big]} - 4\sqrt{\frac{\log(1/\delta)}{3}}.
    \end{equation}
    Combine equations (\ref{con.3.2}) and (\ref{con.3.3}) then the result follows.
\end{proof}
\begin{proof}[Proof of Theorem \ref{aux.4}]
Combine Lemmas \ref{aux.2} and  \ref{aux.3}. For any $\delta > 0$, with probability at least $1 - 2(\log\lfloor n\rfloor+3)\delta$,
\[
    \bigg|\sum_{i=1}^nX_i\bigg| \leq 2\sqrt{\sum_{i = 1}^n\E[X_i^2|X_{<i}]\log(2/\delta)} +\frac{2c}{3} \log(2/\delta) + c\sqrt{\log(2/\delta)},
\]
and
\[
    \sqrt{\sum_{i = 1}^{n}\E\Big[X^2_i\Big|X_{<i}\Big]} \leq \sqrt{\sum_{i = 1}^{n}X^2_{i}} + 4\sqrt{\frac{c\log(1/\delta)}{3}}.
\]
hold at the same time, where we apply Lemma \ref{aux.3} with $Y_i = X_i^2$. In this case,
\[
    \begin{aligned}
        \bigg|\sum_{i=1}^nX_i\bigg| &\leq 2\sqrt{\sum_{i = 1}^nX^2_{i}\log(2/\delta)} +
8\sqrt{\frac{c\log(1/\delta)\log(2/\delta)}{3}} +\frac{2c}{3} \log(2/\delta) + c\sqrt{\log(2/\delta)}\\
&\leq 2\sqrt{\sum_{i = 1}^nX^2_{i}\log(2/\delta)} + 7\max\{c, 1\}\log(2/\delta),
    \end{aligned}
\]
because $1 +2/3 + 8/\sqrt{3}< 7$ and $1 - 2(\lfloor\log n\rfloor+3)\delta \geq 0$ implies $\delta < 1/6$ and $\log(1/\delta) > 1$.
\end{proof}

\section{A good event}
\label{app.goe}

Define an event
 \beqnn
 \mathcal{E} := \Bigg\{\big|\big(\widehat{P}^k_h - P_h\big)V_{h + 1}^*(s, a)\big|
 &&\leq 8\sqrt{\frac{\widehat{\V}^k(V_{h+1}^*, s, a)\eta_1}{C^k_h(s, a)}}+ \frac{H\eta_4}{C^k_h(s, a)}+ \sigma D_{\delta},
 \\ &&\qquad\qquad\qquad\forall\ (s, a, h, k) \in \mathcal{S} \times \mathcal{A} \times [H] \times [K]\Bigg\},
\eeqnn
where $ \eta_1 = C_1\eta\;\hbox{and}\;\eta_4=16\eta_1 + 14\max\{C_1, 1\}\eta + 2\beta.$
\begin{lemma}
    \label{goe}
    $$\pr(\mathcal{E}) \geq 1 - 2|\mathcal{C}_{\sigma^3/KH^2}|HK(\lfloor\log_2(HK)\rfloor + 3)\delta.$$
\end{lemma}
\begin{proof}
 For notation convenience, denote by $V_h = V_{h}^*$.
 We divide the proof into two steps.
 \bmhead{Step 1} We prove that for any given $(s, a) \in \mathcal{S} \times \mathcal{A}$, with probability at least $1 - \delta$,
 \begin{align}
\big|\big(\widehat{P}^k_h - P_h\big)V_{h + 1}(s, a)\big| \leq 8\sqrt{\frac{\widehat{\V}^k(V_{h+1}, s, a)\eta_1}{C^k_h(s, a)}} + \frac{H\eta_4}{C^k_h(s, a)} + \sigma\eta_3
    \label{lem.goe.1}
    \end{align}
holds for all $(h, k) \in [H] \times [K]$.
 By the definition of $\widehat{P}^k$,
\begin{align}
\big|\big(\widehat{P}^k_h - P_h\big)V_{h+1}(s, a)\big| &\leq \bigg|\sum_{l = 1}^{k - 1}\widetilde{\mathnormal{w}}_{h}^{k, l}(s, a)\big(V_{h+1}(s^l_{h + 1}) - P_hV_{h+1}(s, a)\big)\bigg| + \Big|\frac{\beta P_hV_{h+1}(s, a)}{C^k_h(s, a)}\Big|\notag\\
        &\leq \bigg|\sum_{l = 1}^{k - 1}\widetilde{\mathnormal{w}}_{h}^{k, l}(s, a)\big(V_{h+1}(s^l_{h + 1})-P_hV_{h+1}(s, a)\big)\bigg| + \frac{H\beta}{C^k_h(s, a)}.
    \label{app.goe.1}
    \end{align}
    Recall $V_{h+1}$ is $L_1$-Lipschitz and $\widetilde{\mathnormal{w}}_{h}^{k, l}(s, a) = \frac{{\mathnormal{w}}_{h}^{l}(s, a)}{C^k_h(s, a)}$, it follows that
    \begin{align}
        &\bigg|\sum_{l = 1}^{k - 1}\widetilde{\mathnormal{w}}_{h}^{k, l}(s, a)\big(V_{h+1}(s^l_{h + 1}) - P_hV_{h+1}(s, a)\big)\bigg|\notag\\
        &\leq \bigg|\sum_{l = 1}^{k - 1}\widetilde{\mathnormal{w}}_{h}^{k, l}(s, a)\big(V_{h+1}(s^l_{h + 1}) - P_hV_{h+1}(s_{h}^l, a_h^l)\big)\bigg| + \bigg|\sum_{l = 1}^{k - 1}\widetilde{\mathnormal{w}}_{h}^{k, l}(s, a)\big(P_hV_{h+1}(s_{h}^l, a_h^l) - P_hV_{h+1}(s, a)\big)\bigg|\notag\\
        &\leq \frac{1}{C^k_h(s, a)}\bigg|\sum_{l = 1}^{k - 1}{\mathnormal{w}}_{h}^{l}(s, a)\big(V_{h+1}(s^l_{h + 1}) - P_hV_{h+1}(s_{h}^l, a_h^l)\big)\bigg| + \lambda_pL_1\Delta^k_h(s, a),\label{app.goe.2}
    \end{align}
    where $\Delta^k_h(s, a) = \sum_{l = 1}^{k - 1}\widetilde{\mathnormal{w}}_{h}^{k, l}(s, a)\rho\big[(s, a), (s_h^l, a_h^l)\big]$.
    For the first term define $$W_l = {\mathnormal{w}}_{h}^{l}(s, a)\big(V_{h+1}(s^{l}_{h + 1}) - PV_{h+1}(s_{h}^{l}, a_{h}^{l})\big),$$ then $\{W_l\}_l$ is a martingale w.r.t. $\{\mathcal{F}_{h}^l\}_{l}$. Note that $|W_l| \leq C_1H$, by Theorem \ref{aux.4}, for any given $(s, a) \in \mathcal{S} \times \mathcal{A}$, with probability at least $1 - 2KH(\lfloor\log_2(KH)\rfloor + 3)\delta$,
    \begin{align}
        &\Bigg|\sum_{l = 1}^{k - 1}{\mathnormal{w}}_{h}^{l}(s, a)\big(V_{h+1}(s^l_{h + 1})- P_hV_{h+1}(s_{h}^l, a_h^l)\big)\Bigg| = \Bigg|\sum_{l = 1}^{k - 1}W_l\Bigg| \leq 2\sqrt{\eta\sum_{l = 1}^{k - 1}W_l^2} + 7\max\{C_1H, 1\}\eta\notag\\
        &= 2\sqrt{\eta\sum_{l = 1}^{k - 1}\Big[{\mathnormal{w}}_{h}^{l}(s, a)\big(V_{h+1}(s^l_{h + 1}) - P_hV_{h+1}(s_{h}^l, a_h^l)\big)\Big]^2} + 7\max\{C_1H, 1\}\eta\label{app.goe.3}.
    \end{align}
    holds for all $(k, h) \in [K] \times [H]$. Recall that $V_h$ is $L_1$-Lipschitz. It is easy to get that
    \begin{align}
        &\sum_{l = 1}^{k - 1}\Big[{\mathnormal{w}}_{h}^{l}(s, a)\big(V_{h+1}(s^l_{h + 1}) - PV_{h+1}(s_{h}^l, a_h^l)\big)\Big]^2\notag\\
        &\leq 2\sum_{l = 1}^{k - 1}\Big[{\mathnormal{w}}_{h}^{l}(s, a)\big(V_{h+1}(s^l_{h + 1}) - P_hV_{h+1}(s, a)\big)\Big]^2 + 2\sum_{l = 1}^{k - 1}\Big[{\mathnormal{w}}_{h}^{l}(s, a)\big(P_hV_{h+1}(s, a) - P_hV_{h+1}(s_{h}^l, a_h^l)\big)\Big]^2\notag\\
        &\leq 2\sum_{l = 1}^{k - 1}\Big[{\mathnormal{w}}_{h}^{l}(s, a)\big(V_{h+1}(s^l_{h + 1}) - P_hV_{h+1}(s, a)\big)\Big]^2 + 2\lambda_p^2L_1^2\sum_{l = 1}^{k - 1}\Big[{\mathnormal{w}}_{h}^{l}(s, a)\rho\big[(s, a), (s_{h}^l, a_h^l)\big]\Big]^2.\label{app.goe.4}
    \end{align}
    Observe that
    \[
        \sqrt{\frac{\sum_{l = 1}^{k - 1}\Big[{\mathnormal{w}}_{h}^{l}(s, a)\rho\big[(s, a), (s_{h}^l, a_h^l)\big]\Big]^2}{\big[C^k_h(s, a)\big]^2}} \leq \frac{\sum_{l = 1}^{k - 1}{\mathnormal{w}}_{h}^{l}(s, a)\rho\big[(s, a), (s_{h}^l, a_h^l)\big]}{C^k_h(s, a)} = \Delta^k_h(s, a),
    \]
    and that
    \begin{align}
        &\sum_{l = 1}^{k - 1}\Big[{\mathnormal{w}}_{h}^{l}(s, a)\big(V_{h+1}(s^l_{h + 1}) - P_hV_{h+1}(s, a)\big)\Big]^2\notag\\
        &\leq 2\sum_{l = 1}^{k - 1}\Big[{\mathnormal{w}}_{h}^{l}(s, a)\big(V_{h+1}(s^l_{h + 1}) - \widehat{P}^kV_{h+1}(s, a)\big)\Big]^2 + 2\sum_{l = 1}^{k - 1}\Big[{\mathnormal{w}}_{t}^{l}(s, a)(\widehat{P}^k - P)V_{h+1}(s, a)\Big]^2\notag\\
        &\leq 2C_1\sum_{l = 1}^{k - 1}{\mathnormal{w}}_{h}^{l}(s, a)\big[V_{h+1}(s^l_{h + 1}) - \widehat{P}^k_hV_{h+1}(s, a)\big]^2 + 2C_1\sum_{l = 1}^{k - 1}{\mathnormal{w}}_{h}^{l}(s, a)[(\widehat{P}^k_h - P_h)V_{h+1}(s, a)]^2\notag\\
        &\leq 2C_1C^k_h(s, a)\Big[\sum_{l = 1}^{k - 1}\widetilde{\mathnormal{w}}_{h}^{k, l}(s, a)\big(V_{h+1}(s^l_{h + 1}) - \widehat{P}^kV_{h+1}(s, a)\big)^2 + \sum_{l = 1}^{k - 1}\widetilde{\mathnormal{w}}_{h}^{k, l}(s, a)[(\widehat{P}^k_h - P_h)V_{h+1}(s, a)]^2\Big]\notag\\
        &\leq 2C_1C^k_h(s, a)\Big[\widehat{\V}^k_h(V_{h+1}, s, a)+H\big|\big(\widehat{P}^k - P\big)V_{h+1}(s, a)\big|\Big].
        \label{app.goe.6}
    \end{align}
    where the second inequality holds since ${\mathnormal{w}}_{h}^{l}(s, a) \leq C_1$ and the final inequality applies $\sum_{l = 1}^{k - 1}\sum_{t = 1}^{H}\widetilde{\mathnormal{w}}_{h}^{k, l}(s, a) \leq 1$.

    Combining equations (\ref{app.goe.2}), (\ref{app.goe.3}), (\ref{app.goe.4}) and (\ref{app.goe.6}) yields that
    \begin{align}
        &\bigg|\sum_{l = 1}^{k - 1}\widetilde{\mathnormal{w}}_{h}^{k, l}(s, a)\big(V_{h+1}(s^l_{h + 1}) - P_hV_{h+1}(s, a)\big)\bigg|\notag\\
        &\leq 4\sqrt{\frac{\widehat{\V}^k_h(V_{h+1}, s, a)\eta_1}{C^k_h(s, a)}} +  4\sqrt{\frac{H\big|\big(\widehat{P}^k - P\big)V_{h+1}(s, a)\big|\eta_1}{C^k_h(s, a)}}\notag\\
        &\qquad\qquad\qquad\qquad\qquad\qquad\qquad\qquad+ \lambda_pL_1(1 + 3\sqrt{\eta})\Delta^k(s, a) + \frac{7\max\{C_1H, 1\}\eta}{C^k_h(s, a)},\label{app.goe.7}
    \end{align}
    where $\eta_1 = C_1\eta$. From equations (\ref{app.goe.1}) and (\ref{app.goe.7}), it follows that for any given $(s, a) \times \mathcal{S} \times \mathcal{A}$ and any $\delta \in (0, 2e^{-1})$, with probability at least $1 - 2KH(\lfloor\log_2(HK)\rfloor + 3)\delta$,
    \begin{align}
        &\big|\big(\widehat{P}^k_h - P_h\big)V_{h+1}(s, a)\big|\notag\\ &\leq 4\sqrt{\frac{\widehat{\V}^k_h(V_{h+1}, s, a)\eta_1}{C^k_h(s, a)}} +  4\sqrt{\frac{H\big|\big(\widehat{P}^k_h - P_h\big)V_{h+1}(s, a)\big|\eta_1}{C^k(s, a)}} + 4\lambda_pL_1\eta\Delta^k(s, a) + \frac{H\eta_2}{C^k_h(s, a)}
    \end{align}
    holds for all $(h, k) \in
    [H] \times [K]$, where $\eta_2 = 7\max\{C_1, 1\}\eta + \beta$. Therefore,
    $$
     \Bigg(\sqrt{\big|\big(\widehat{P}^k_h - P_h\big)V_{h+1}(s, a)\big|} - 2\sqrt{\frac{H\eta_1}{C^k_h(s, a)}}\Bigg)^2 \leq 4\sqrt{\frac{\widehat{\V}^k_h(V_{h+1}, s, a)\eta_1}{C^k_h(s, a)}} + 4\lambda_pL_1\eta\Delta^k_h(s, a) + \frac{(4\eta_1 + \eta_2)H}{C^k_h(s, a)} ,$$
     which implies that
     \begin{align}
        \label{app.goe.8}
         \big|\big(\widehat{P}^k_h - P_h\big)V_{h+1}(s, a)\big| \leq 8\sqrt{\frac{\widehat{\V}^k(V_{h+1}, s, a)\eta_1}{C^k_h(s, a)}} + 8\lambda_pL_1\eta\Delta^k_h(s, a) + \frac{(16\eta_1 + 2\eta_2)H}{C^k_h(s, a)},
     \end{align}
    where we use the inequality $(a + b)^2 \leq 2a^2 + 2b^2$. Additionally, Lemma \ref{aux.5} shows that
    \begin{equation}
        \Delta^k_h(s, a) = \sum_{l = 1}^{k - 1}\widetilde{\mathnormal{w}}_{h}^{k, l}(s, a)\rho\big[(s, a), (s_h^l, a_h^l)\big] \leq \sigma\gamma.\label{lem7DK}
    \end{equation}
  Let $\eta_3=8\lambda_pL_1\eta\gamma$ and $ \eta_4 = 16\eta_1 + 2\eta_2$ , substituting (\ref{lem7DK}) into (\ref{app.goe.8}) yields (\ref{lem.goe.1}).
\bmhead{Step 2} We extend the above claim to all state-action pairs $(s, a) \in \mathcal{S} \times \mathcal{A}$ with necessary modifications by covering argument. Let $\sigma' = \sigma^3/(KH^2)$ and $\mathcal{C}_{\sigma'}$ be the
 $\sigma$-cover set of $\mathcal{S} \times \mathcal{A}$. For any $(s ,a) \in \mathcal{S} \times \mathcal{A}$, denote
\[
    (s', a') = \arg\min_{(s_j, a_j) \in \mathcal{C}_{\sigma'}}\rho[(s, a), (s_j, a_j)],
\]
where ties are broken uniformly. Then
\[
    \rho[(s, a), (s', a')] \leq \frac{\sigma^3}{KH^2},
\]
and
\begin{align}
    \label{app.goe.decomp}
    \big|\big(\widehat{P}^k_h - P_h\big)V_{h + 1}(s, a)\big| \leq \big|\big(\widehat{P}^k_h - P_h\big)\big( V_{h+1}(s, a) - V_{h+1}(s', a')\big)\big| + \big|\big(\widehat{P}^k_h - P_h\big)V_{h+1}(s', a')\big|.
\end{align}
Note that
\[
    \big|\big(\widehat{P}^k_h - P_h\big)V_{h+1}(s', a')\big| \leq \sum_{(s_j, a_j) \in \mathcal{C}_{\sigma'}}\big|\big(\widehat{P}^k_h - P_h\big)V_h(s_j, a_j)\big|\mathbf{1}[(s', a') = (s_j, a_j)]
\]
holds uniformly for each $(s, a) \in \mathcal{S} \times \mathcal{A}$. By the simple union, we have that with probability at least $1 - 2|\mathcal{C}_{\sigma'}|HK(\lfloor\log_2(HK)\rfloor + 3)\delta$, (\ref{lem.goe.1}) holds for all $(h, k) \in [H] \times [K]$ and all $(s_j, a_j) \in \mathcal{C}_{\sigma'}$. Therefore
with probability at least $1 - 2|\mathcal{C}_{\sigma'}|HK(\lfloor\log_2(HK)\rfloor + 3)\delta$,
\begin{align}
    &\qquad\big|\big(\widehat{P}^k_h - P_h\big)V_{h+1}(s', a')\big|\\
    &\leq \sum_{(s_j, a_j) \in \mathcal{C}_{\sigma'}}\Bigg[8\sqrt{\frac{\widehat{\V}^k_h(V_{h+1}, s_j, a_j)\eta_1}{C^k_h(s_j, a_j)}} + \frac{H\eta_4}{C^k_h(s_j, a_j)}, + \sigma\eta_3\Bigg]\mathbf{1}[(s', a') = (s_j, a_j)] \notag\\
    &\leq 8\sqrt{\frac{\widehat{\V}^k_h(V_{h+1}, s, a)\eta_1}{C^k_h(s, a)}} + \frac{H\eta_4}{C^k_h(s, a)} + \sigma\eta_3\notag\\
    &\qquad+\sum_{(s_j, a_j) \in \mathcal{C}_{\sigma'}}\Bigg\{8\Bigg[\sqrt{\frac{\widehat{\V}^k_h(V_{h
+1}, s_j, a_j)\eta_1}{C^k_h(s_j, a_j)}} - \sqrt{\frac{\widehat{\V}^k_h(V_{h+1}, s, a)\eta_1}{C^k_h(s, a)}}\Bigg]\notag\\
    &\qquad+ \Bigg[\frac{1}{C^k_h(s_j, a_j)} - \frac{1}{C^k_h(s, a)}\Bigg]H\eta_4\Bigg\}\mathbf{1}_{(s', a') = (s_j, a_j)}.
\end{align}
Note that for any $k, h$ and any $(s, a) \in \mathcal{S} \times \mathcal{A}$,
\begin{align}
    \label{app.goe.9}
    &\Bigg|\sqrt{\frac{\widehat{\V}^k(V_{h
+1}, s_j, a_j)\eta_1}{C^k_h(s_j, a_j)}} - \sqrt{\frac{\widehat{\V}^k_h(V_{h+1}, s, a)\eta_1}{C^k(s, a)}}\Bigg| \leq \sqrt{\eta_1}\sqrt{\Bigg|\frac{\widehat{\V}^k_h(V_{h
+1}, s_j, a_j)}{C^k_h(s_j, a_j)} - \frac{\widehat{\V}^k_h(V_{h+1}, s, a)}{C^k_h(s, a)}\Bigg|}\notag\\
    &\leq \sqrt{\eta_1}\sqrt{\Bigg|\frac{\widehat{\V}^k_h(V_{h+1}, s_j, a_j) - \widehat{\V}^k_h(V_{h+1}, s, a)}{C^k_h(s_j, a_j)}\Bigg| + \widehat{\V}^k_h(V_{h+1}, s, a)\Bigg|\frac{1}{C^k_h(s_j, a_j)} - \frac{1}{C^k_h(s, a)}\Bigg|}\notag\\
    &\leq \sqrt{\eta_1}\sqrt{\Bigg|\frac{\widehat{\V}^k_h(V_{h+1}, s_j, a_j) - \widehat{\V}^k_h(V_{h+1}, s, a)}{\beta}\Bigg| + H^2\Bigg|\frac{1}{C^k_h(s_j, a_j)} - \frac{1}{C^k_h(s, a)}\Bigg|},
\end{align}
where we use the facts that $C^k_h(s, a) \geq \beta$ and $\widehat{\V}^k(V_{h+1}, s, a) \leq H^2$. Moreover, using the definition
of $\widehat{\V}^k_h$ and inequalities (\ref{app.aux.1}) and (\ref{app.aux.2}) we have that
\begin{align}
    \label{app.goe.10}
    \Bigg|\widehat{\V}^k_h(V_{h+1}, s_j, a_j) - \widehat{\V}^k_h(V_{h+1}, s, a)\Bigg|\leq 3C_2H^2k\frac{\rho[(s, a), (s_j, a_j)]}{\beta^2\sigma},
\end{align}
and
\begin{align}
    \label{app.goe.11}
    \Bigg|\frac{1}{C^k(s_j, a_j)} - \frac{1}{C^k(s, a)}\Bigg| \leq C_2k\frac{\rho[(s, a), (s_j, a_j)]}{\beta^2\sigma}.
\end{align}
Substituting (\ref{app.goe.10}) and (\ref{app.goe.11}) into (\ref{app.goe.9}) yields that
\begin{align}
    \label{app.goe.12}
    \Bigg|\sqrt{\frac{\widehat{\V}^k_h(V_{h+1}, s_j, a_j)\eta_1}{C^k_h(s_j, a_j)}} - \sqrt{\frac{\widehat{\V}^k_h(V_{h+1}, s, a)\eta_1}{C^k_h(s, a)}}\Bigg| \leq 2\sqrt{\frac{C_2\eta_1 H^2K\rho[(s, a), (s_j, a_j)]}{\beta^3\sigma}}.
\end{align}
Hence
\begin{align}
    &\sum_{(s_j, a_j) \in \mathcal{C}_{\sigma'}}\Bigg\{8\Bigg[\sqrt{\frac{\widehat{\V}^k_h(V_{h+1}, s_j, a_j)\eta_1}{C^k_h(s_j, a_j)}} - \sqrt{\frac{\widehat{\V}^k_h(V_{h+1}, s, a)\eta_1}{C^k_h(s, a)}}\Bigg]\\
    &\qquad\qquad\qquad\qquad\qquad\qquad\qquad\qquad\qquad\qquad+ \Bigg[\frac{1}{C^k_h(s_j, a_j)} - \frac{1}{C^k_h(s, a)}\Bigg]H\eta_4\Bigg\}\mathbf{1}_{(s', a') = (s_j, a_j)}\notag\\
    &\qquad\leq \sum_{(s_j, a_j) \in \mathcal{C}_{\sigma'}}\Bigg\{16\sqrt{\frac{C_2\eta_1 H^2K\rho[(s, a), (s_j, a_j)]}{\beta^3\sigma}}+ \frac{C_2\eta_4HK\rho[(s, a), (s_j, a_j)]}{\beta^2\sigma}\Bigg\}\mathbf{1}_{(s', a') = (s_j, a_j)}\notag\\
    &\qquad\leq \sum_{(s_j, a_j) \in \mathcal{C}_{\sigma'}}\Bigg\{16\sqrt{\frac{C_2\eta_1\eta\sigma^2}{\beta^3}} + \frac{C_2\eta_4\sigma^2}{H\beta^2}\Bigg\}\mathbf{1}_{(s', a') = (s_j, a_j)}\notag\\
    &\qquad\leq 16\sqrt{\frac{C_2\eta_1\eta\sigma^2}{\beta^3}}+ \frac{C_2\eta_4\sigma^2}{H\beta^2}.\label{app.goe.13}
\end{align}
In addition, since $0 \leq V_{h+1} \leq H$, inequality (\ref{app.aux.1}) and Assumption \ref{assu.2} show that
\begin{align}
    &\big|\big(\widehat{P}^k_h - P_h\big)\big( V_{h+1}(s, a) - V_{h+1}(s', a')\big)\big| \notag
    \\&\leq  \big|\widehat{P}^k_h\big( V_{h+1}(s, a) - V_{h+1}(s', a')\big)\big| +\big|P_h\big( V_{h+1}(s, a) - V_{h+1}(s', a')\big)\big|\notag\\
    &\leq 2C_2\frac{\sigma^2}{H\beta^2} + \lambda_pL_1\frac{\sigma^3}{KH^2}.\label{app.goe.14}
\end{align}
Let
\[
    D_{\delta} = 16\sqrt{\frac{C_2\eta_1\eta}{\beta^3}} + \big(C_2\eta_4 + 2C_2\big)\frac{\sigma}{H\beta^2} + \lambda_pL_1\frac{\sigma^2}{KH^2} + \eta_3.
\]
Combining (\ref{app.goe.decomp}) and (\ref{app.goe.8}), (\ref{app.goe.13}), (\ref{app.goe.14}) yields the desired results.
\end{proof}

\section{Optimism}

\begin{lemma}
    \label{opt}
    On $\mathcal{E}$, $V^k_h \geq V^*_h$ for all $(h, k) \in [H + 1] \times [K]$.
\end{lemma}

\label{app.opt}
The proof of this claim relies on the property of $\mathcal{T}_{p, c, l}$ defined in the following lemma.
\begin{lemma}
    \label{opt.1}
    For any function $v$ on $\mathcal{S}$ which takes values in $[0, v_{\max}]$, let $$\mathcal{T}_{p, c,l}(v) = p\mathnormal{v} + \max\Big\{c_1\sqrt{\frac{\mathbb{V}(p, \mathnormal{v})l}{c}},  c_2\frac{v_{\max}l}{c}
    \Big\}$$ with $c_1 = 9, c_2 = 162$, where $p$ is a measure on $\mathcal{S}$, $c \in \mathbb{R}_{> 0}$, $pv = \int vdp$ and $\mathbb{V}(p, v) = p[(v - pv)^2]$. Then for each $p$ and $c, l > 0$, $\mathcal{T}_{p, c, l}$ satisfies
    \begin{itemize}
        \item[1.] $\mathcal{T}_{p, c, l}$ is no-decreasing in $v$, that is, if $v_1 \geq v_2 \geq 0$, then $\mathcal{T}_{p, c, l}(v_1) \geq \mathcal{T}_{p, c, l}(v_2)$;
        \item[2.] $\mathcal{T}_{p, c,l}(v) \geq pv + 8\sqrt{\frac{\mathbb{V}(p, v)l}{c}} + \frac{ v_{\max}l}{c}$.
    \end{itemize}
\end{lemma}
\begin{proof}
    We begin at the first claim, suppose that $v_1 \geq v_2 \geq 0$, then we have
    \[
        \begin{aligned}
            \mathcal{T}_{p, c, l}(v_1) - \mathcal{T}_{p, c, l}(v_2) &\geq p(v_1 - v_2) + {c}_1\mathbf{1}\Bigg\{{c}_1\sqrt{\frac{\mathbb{V}(p, \mathnormal{v}_2)l}{c}} \geq {c}_2\frac{v_{\max}l}{c}\Bigg\}\Bigg(\sqrt{\frac{\mathbb{V}(p, \mathnormal{v}_1)l}{c}} - \sqrt{\frac{\mathbb{V}(p, \mathnormal{v}_2)l}{c}}\Bigg)\\
            &= p(v_1 - v_2) + {c}_1\sqrt{\frac{l}{c}}\mathbf{1}\Bigg\{{c}_1\sqrt{\frac{\mathbb{V}(p, \mathnormal{v}_2)l}{c}} \geq {c}_2\frac{v_{\max}l}{c}\Bigg\}\frac{\mathbb{V}(p, \mathnormal{v}_1) - \mathbb{V}(p, \mathnormal{v}_2)}{\sqrt{\mathbb{V}(p, \mathnormal{v}_1)} + \sqrt{\mathbb{V}(p, \mathnormal{v}_2)}}.
        \end{aligned}
        \]
        If $\mathbb{V}(p, \mathnormal{v}_1) \geq  \mathbb{V}(p, \mathnormal{v}_2)$,  then the claim (1) holds. Below we assume $\mathbb{V}(p, \mathnormal{v}_1) <  \mathbb{V}(p, \mathnormal{v}_2)$. In this case, we have that
        \beqnn
        &&c_1\sqrt{\frac{l}{c}}\mathbf{1}\Bigg\{{c}_1\sqrt{\frac{\mathbb{V}(p, \mathnormal{v}_2)l}{c}} \geq {c}_2\frac{v_{\max}l}{c}\Bigg\}\frac{\mathbb{V}(p, \mathnormal{v}_1) - \mathbb{V}(p, \mathnormal{v}_2)}{\sqrt{\mathbb{V}(p, \mathnormal{v}_1)} + \sqrt{\mathbb{V}(p, \mathnormal{v}_2)}}
        \\&&\geq {c}_1\sqrt{\frac{l}{c}}\mathbf{1}\Bigg\{\sqrt{\mathbb{V}(p, \mathnormal{v}_2)}\geq \frac{c_2v_{\max}}{c_1}\sqrt{\frac{l}{c}}\Bigg\}\frac{\mathbb{V}(p, \mathnormal{v}_1) - \mathbb{V}(p, \mathnormal{v}_2)}{ \sqrt{\mathbb{V}(p, \mathnormal{v}_2)}}
        \\&& \geq \frac{c_1^2}{c_2v_{\max}}\mathbf{1}\Bigg\{\sqrt{\mathbb{V}(p, \mathnormal{v}_2)}\geq \frac{c_2v_{\max}}{c_1}\sqrt{\frac{l}{c}}\Bigg\}\Big(\mathbb{V}(p, \mathnormal{v}_1) -\mathbb{V}(p, \mathnormal{v}_2)\Big).
        \eeqnn
        Therefore
        \[
        \begin{aligned}
         \mathcal{T}_{p, c, l}(v_1) - \mathcal{T}_{p, c, l}(v_2) &\geq  p(v_1 - v_2) + \frac{{c}^2_1}{{c}_2v_{\max}}\Big(\mathbb{V}(p, \mathnormal{v}_1) - \mathbb{V}(p, \mathnormal{v}_2)\Big)\\
            &\geq  p(v_1 - v_2) + \frac{{c}^2_1}{{c}_2v_{\max}}\Big(p(v_1 + v_2)(v_1 - v_2) - p(v_1 + v_2)p(v_1 - v_2)\Big)\\
            &\geq  p(v_1 - v_2) - \frac{2{c}^2_1}{{c}_2}p(v_1 - v_2)= p(v_1 - v_2)\bigg(1 - \frac{2{c}^2_1}{{c}_2}\bigg)= 0.
        \end{aligned}
    \]
    The third and the fourth inequalities hold because of $v_1 \geq v_2\geq 0$ and
    the last equation holds since $c_2 = 2c_1^2$.

   It is easy to check that
    \beqnn
    \max\Big\{9\sqrt{\frac{\mathbb{V}(p, \mathnormal{v})l}{c}},  162\frac{v_{\max}l}{c}
    \Big\}\geq\begin{cases} 8\sqrt{\frac{\mathbb{V}(p, \mathnormal{v})l}{c}}+\frac{162}{9}\times\frac{v_{\max}l}{c}= 8\sqrt{\frac{\mathbb{V}(p, \mathnormal{v})l}{c}}+18\frac{lv_{\max}}{c}, &9\sqrt{\frac{\mathbb{V}(p, \mathnormal{v})l}{c}}>162\frac{lv_{\max}}{c},
    \\ \frac{161}{162}\times 9\sqrt{\frac{\mathbb{V}(p, \mathnormal{v})l}{c}}+\frac{lv_{\max}}{c}>8\sqrt{\frac{\mathbb{V}(p, \mathnormal{v})l}{c}}+\frac{lv_{\max}}{c}, &11\sqrt{\frac{\mathbb{V}(p, \mathnormal{v})l}{c}}\leq 242\frac{lv_{\max}}{c}.
    \end{cases}
    \eeqnn
   The second claim holds.
\end{proof}
\begin{proof}[Proof of Lemma \ref{opt}]
     The claim is true when $h = H + 1$ since $V^k_{H + 1}(s) = V^*_{H + 1}(s) = 0$. Now consider $1 \leq h \leq H$. Suppose that the claim holds for $h + 1$ on the good event $\mathcal{E}$. Recall we choose $c_1 = 9, c_2 = 162$, it follows that for any $(s, a) \in \mathcal{D}$ the agent visited and any $k \in [K]$,
    \[
        \begin{aligned}
            \widetilde{Q}_{h}^k(s, a)&=  r_h(s, a) + \widehat{P}^k_hV^{k}_{h + 1}(s, a) + b^k_{h}(s, a)\\
            &= r_h(s, a) + \widehat{P}^k_hV^k_{h+ 1}(s, a) + c_1\sqrt{\frac{\widehat{\V}^k(V_{h + 1}^k, s, a)\eta_4}{C^k_h(s, a)}} + \frac{c_2H\eta_4}{C^k(s, a)} + \sigma D_{\delta}\\
            &\geq r_h(s, a) + \widehat{P}^k_hV^k_{h + 1}(s, a) + \max\Bigg\{{c}_1\sqrt{\frac{\widehat{\V}^k_h(V_{h + 1}^k, s, a)\eta_4}{C^k_h(s, a)}}, \frac{{c}_2H\eta_4}{C^k_h(s, a)}\Bigg\}+ \sigma D_{\delta}\\
            &\geq r_h(s, a) + \widehat{P}^kV^*_{h + 1}(s, a) + \max\Bigg\{{c}_1\sqrt{\frac{\widehat{\V}^k_h(V_{h + 1}^*, s, a)\eta_4}{C^k_h(s, a)}}, \frac{{c}_2H\eta_4}{C^k_h(s, a)}\Bigg\}+ \sigma D_{\delta}\\
            &\geq r_h(s, a) + \widehat{P}_h^kV^*_{h + 1}(s, a) + 8\sqrt{\frac{\widehat{\V}^k_h(V^*_{h+1}, s, a)\eta_1}{C^k_h(s, a)}} + \frac{H\eta_4}{C^k_h(s, a)} + \sigma D_{\delta}\\
            &\geq r_h(s, a) + P_hV^*_{h + 1}(s, a)\\
            &=Q^*_h(s, a),
        \end{aligned}
    \]
    where the first and the second equations are based on the definition of $\widetilde{Q}$ and $b_h^k$, the second inequality follows from the first claim of Lemma \ref{opt.1}
    . The third inequality holds due to the second assertion of Lemma \ref{opt.1}, along with the definitions of $\eta_1$ and $\eta_4$. The last inequality is from the definition of $\mathcal{E}$ and Lemma \ref{goe}. It follows that for all $(s, a, k, h, l) \in \mathcal{S} \times \mathcal{A} \times [K] \times [H] \times [k - 1]$,
    \[
        \begin{aligned}
            \widetilde{Q}_h^k(s_{h}^l, a_h^l) + L_h\rho[(s, a), (s_{h}^l, a_h^l)] &\geq Q^*_h(s_{h}^l, a_h^l) + L_h\rho[(s, a), (s_{h}^l, a_h^l)]\geq Q^*_h(s, a).
        \end{aligned}
    \]
    Therefore, $Q^k_h(s, a) \geq Q^*_{h}(s, a)$, which implies $V^k_h(s) \geq V^*_{h}(s)$.
\end{proof}

\section{Concentration}

\label{app.con}

The theoretical analysis of the regret requires to bound $\big(\widehat{P}^k_h - P_h\big)f(s, a)$ for some bounded Lipschitz function $f$. The parameters $L$ and $F$ are flexible in the following lemma
\begin{lemma}
    \label{concentration}
    The following event $\mathcal{E}_{L, F}$ holds with probability at least $1 -2|\mathcal{C}_{\sigma^3/(D_{L,F,\sigma}KHF^2)}|HK(\lfloor \log_2 (HK)\rfloor+3)\delta$. Here
    \begin{align*}
        \mathcal{E}_{L, F} := \Bigg\{\big|\big(\widehat{P}^k_h - P_h\big)f(s, a)\big|
        &\leq 2\sqrt{\frac{\V_h(f, s, a)D_{L,F,\sigma}}{C^k_h(s, a)}}+\frac{2FD_{L,F,\sigma}}{C^k_h(s, a)}+ \sigma\widetilde{D}_{L, F},
        \\&\qquad\qquad\qquad\qquad\text{for all}\ (s, a, k) \in \mathcal{S} \times \mathcal{A} \times [K], f \in \mathcal{F}_{L, F}\Bigg\},
    \end{align*}

where $$\mathcal{F}_{L, F}= \{f : \mathcal{S} \to [0, F]\; \text{is}\; L \text{-Lipschitz w.r.t.}\ \rho_{\mathcal{S}}\},$$
and $\widetilde{D}_{L, F} = D_{L,F}+12L, D_{L,F,\sigma}=\eta_5+|\widetilde{\mathcal{C}}_{\sigma}|\log(\frac{2F}{L\sigma})$ are two constants. The definition of $\eta_5$ and $D_{L,F}$ can be found in Lemma \ref{con.1}
\end{lemma}
The proof of Lemma \ref{concentration} relies the following lemma.

\begin{lemma}
\label{con.1}
Let $f : \mathcal{S} \to [0, F]$ be any given function with Lipschitz constant $L$, where $F$ is an arbitrary given positive constant. The following event $\mathcal{E}_f$ holds with probability at least $1 - 2|\mathcal{C}_{\sigma^3/(\log(2/\delta)KHF^2)}|HK(\lfloor\log_2 (HK)+3)\delta$.
\[
    \begin{aligned}
        \mathcal{E}_{f} = \Bigg\{\big|\big(\widehat{P}^k_h - P_h\big)f(s, a)\big| \leq 2\sqrt{\frac{\V(f, s, a)\eta_1}{C^k(s, a)}} + \frac{F\eta_5}{C^k(s, a)} + \sigma{D}_{L, F}, \forall (s, a, k, h)\Bigg\},
    \end{aligned}
\]
 where
\[
    \begin{aligned}
        \eta_5 &= (2C_1 + 4)\eta + \beta;\\
        D_{L, F} &= 2\sqrt{\frac{4C_1\lambda_pL\sigma}{\beta KF^2} + \frac{C_1C_2}{\beta^2}} + (\eta_5F + 3)\frac{C_2\sigma}{F^2\beta^2} + \lambda_pL\frac{\sigma^2}{KHF^2} + 8\lambda_pL\gamma.
    \end{aligned}
\]
\end{lemma}
\begin{proof}
    Similar to the proof of Lemma \ref{goe}, (see equation (\ref{app.goe.1}) and(\ref{app.goe.2}) and their proofs), we have
    \begin{align}
    \big|\big(\widehat{P}_h^k - P_h\big)f(s, a)\big| \leq \bigg|\sum_{l = 1}^{k - 1}\frac{{\mathnormal{w}}_{h}^{l}(s, a)}{C^k_h(s,a)}\big(f(s^l_{h + 1}) - P_hf(s_{h}^l, a_h^l)\big)\bigg| + \frac{F\beta}{C^k_h(s, a)} + \lambda_pL\Delta^k_h(s, a)\label{app.con.1}.
\end{align}
By the similar discussion leading to (\ref{app.goe.3}), for any given $(s, a, h, k) \in \mathcal{S} \times \mathcal{A} \times
[H] \times [K]$ with probability $1 - 2 HK(\lfloor\log_2 (HK)+3)\delta,$
\begin{align}
    \Bigg|\sum_{l = 1}^{k - 1}{\mathnormal{w}}_{h}^{l}(s, a)\big(f(s^l_{h + 1}) - P_hf(s_{h}^l, a_h^l)\big)\Bigg|
    \leq2\sqrt{\eta\sum_{l = 1}^{k - 1}[w_h^l(s, a)]^2\V_h(f, s_{h}^l, a_h^l)} + 2C_1F\eta.\label{app.con.2}
\end{align}
Since $f\in [0, F]$ is Lipschitz with constant $L$, for any $(s, a), (s', a') \in \mathcal{S} \times \mathcal{A}$,
   \beqlb\label{app.con.3}
            &&|\V_h(f, s', a') - \V_h(f, s, a)|\nonumber\\
            &&\leq |P_hf^2(s', a') - P_hf^2(s, a)| +  \big|[P_hf(s, a)]^2 - [P_hf(s', a')]^2\Big|\nonumber\\
            &&\leq |P_hf^2(s', a') - P_hf^2(s, a)| +  2F|P_hf(s, a) - P_hf(s', a')|\nonumber\\
            &&\leq 4\lambda_pLF\rho\big[(s', a'), (s, a)\big].
    \eeqlb
Therefore, using (\ref{lem7DK}) and noting that $w_t^l(s, a)\in[0, C_1]$, we have that
\beqlb
\label{app.con.4}
    \big|\big(\widehat{P}^k_h - P_h\big)f(s, a)\big| &\leq & 2\sqrt{\frac{\sum_{l = 1}^{k - 1}[w_h^l(s, a)]^2\V_h(f, s, a)\eta}{[C^k_h(s, a)]^2}}+ \frac{2C_1F\eta + F\beta}{C^k_h(s, a)} +  4\lambda_p L\gamma\sigma\nonumber\\
    &&+ 4\sqrt{\frac{\lambda_pLF\sum_{l = 1}^{k - 1}[w_h^l(s, a)]^2\rho\big[(s, a), (s_h^k, a_h^k)\big]\eta}{[C^k_h(s, a)]^2}}\nonumber\\
    &\leq &2\sqrt{\frac{\V_h(f, s, a)\eta_1}{C^k_h(s, a)}} + 8\sqrt{\frac{\sigma\lambda_p L \gamma F\eta}{C^k_h(s, a)}} + \frac{2C_1F\eta + F\beta}{C^k_h(s, a)} +  4\lambda_p L\gamma\sigma\nonumber\\
    &\leq  &2\sqrt{\frac{\V_h(f, s, a)\eta_1}{C^k_h(s, a)}} + \frac{(2C_1 + 4)F\eta + F\beta}{C^k_h(s, a)} +  8\lambda_p L\gamma\sigma,
\eeqlb
where we use the inequality $2\sqrt{ab} \leq a + b$  in the last line.

The remainder is to extend (\ref{app.con.4}) to all
$(s, a) \in \mathcal{S} \times \mathcal{A}$ by the covering arguments which are similar to Step 2 in the proof of Lemma \ref{goe}. We
remind that
\begin{align}
    \label{app.con.5}
    |\V(f, s, a) - \V(f, s', a')| \leq 4\lambda_pLF\rho\big[(s', a'), (s, a)\big].
\end{align}
Following the same arguments leading to (\ref{app.goe.13}), note that $\V(f, s, a) \leq F^2$, we have that for each $(s, a) \in \mathcal{S} \times \mathcal{A}$,
\begin{align}
    &\sum_{(s_i, a_i) \in \mathcal{C}_{\sigma''}}\Bigg\{2\Bigg[\sqrt{\frac{\V_h(f, s_i, a_i)\eta_1}{C^k_h(s_i, a_i)}} - \sqrt{\frac{\V_h(f, s, a)\eta_1}{C^k_h(s, a)}}\Bigg]\notag\\
    &\qquad\qquad\qquad\qquad\qquad+ \Bigg[\frac{(2C_1 + 4)F\eta + F\beta}{C^k_h(s_i, a_i)} - \frac{(2C_1 + 4)F\eta + F\beta}{C^k_h(s, a)}\Bigg]\Bigg\}\mathbf{1}_{(s', a') = (s_i, a_i)}\notag\\
    \leq &2\sqrt{\frac{4C_1\lambda_pL\sigma^3}{\beta KF^2} + \frac{C_1C_2\sigma^2}{\beta^2}} + \big((2C_1 + 4)\eta + \beta\big)\frac{C_2\sigma^2}{\beta^2},\label{app.con.6}
\end{align}
where we use inequality (\ref{app.goe.11}), $\log(2/\delta) \geq 1$,
 $\sigma'' = \frac{\sigma^3}{\log(2/\delta)KF^2}$. In addition, similar to
(\ref{app.goe.14}), we have that
\begin{align*}
    \big|\big(\widehat{P}^k_h - P_h\big)\big( f(s, a) - f(s', a')\big)\big| &\leq  \big|\widehat{P}^k_h\big( f(s, a) - f(s', a')\big)\big| +\big|P_h\big( f(s, a) - f(s', a')\big)\big|\notag\\
    &\leq 2C_2\frac{\sigma^2}{F\beta^2} + \lambda_pL\frac{\sigma^3}{KF^2}.
\end{align*}
Summing up, we get the desired conclusion.
\end{proof}
\begin{proof}[Proof of Lemma \ref{concentration}]
According to Lemma \ref{aux.6}, the $4L\sigma$-covering number of $\mathcal{F}_{L, F}$ is bounded by $(2F/(L\sigma))^{|\widetilde{\mathcal{C}}_{\sigma}|}$, where $\widetilde{\mathcal{C}}_{\sigma} = \mathcal{N}(\sigma, \mathcal{S}, \rho)$. For convenience, denote the corresponding $4L\sigma$-covering set of $\mathcal{F}_{L, F}$ by $\widetilde{\mathcal{C}}$. From Lemma \ref{con.1}, we know that with probability at least $1 - 2|\mathcal{C}_{\sigma^3/(D_{L,F,\sigma}KF^2)}|HK(\lfloor\log_2 (HK)+3)\delta$,
\[
    \begin{aligned}
        \big|\big(\widehat{P}^k_h - P_h\big)f(s, a)\big| \leq 2\sqrt{\frac{\V_h(f, s, a)D_{L,F,\sigma}}{C^k_h(s, a)}} + \frac{FD_{L,F,\sigma}}{C^k_h(s, a)} + \sigma{D}_{L, F}
    \end{aligned}
\]
holds for all $(s, a, k, h) \in \mathcal{S} \times \mathcal{A} \times [K] \times [H]$ and all $f \in \widetilde{\mathcal{C}}$, where $D_{L,F,\sigma} = \eta_5 + |\widetilde{\mathcal{C}}_{\sigma}|\log\frac{2F}{L\sigma}$.\\
For any $L$-Lipschitz function $\mathcal{S} \to [0, F]$, there exists a function $f' \in \widetilde{\mathcal{C}}$ such that $\Vert f - f'\Vert_{\infty} \leq 4L\sigma$, it is easy to see that
\begin{align*}
    \big|\big(\widehat{P}^k_h - P_h\big)f(s, a)\big| \leq \big|\big(\widehat{P}^k_h - P_h\big)f'(s, a)\big| + 2\Vert f - f'\Vert_{\infty}.
\end{align*}
Therefore, in the event $\mathcal{E}_{L, F}$, by the covering argument,
\[
    \begin{aligned}
        \big|\big(\widehat{P}^k_h - P_h\big)f\big| &\leq 2\sqrt{\frac{\V_h(f', s, a)D_{L,F,\sigma}}{C^k_h(s, a)}} + \frac{FD_{L,F,\sigma}}{C^k(s, a)} + \sigma{D}_{L, F} + 8L\sigma\\
        &\leq 2\sqrt{\frac{\V_h(f, s, a)D_{L,F,\sigma}}{C^k_h(s, a)}} + + \frac{2FD_{L,F,\sigma}}{C^k_h(s, a)} + \sigma{D}_{L, F} + 12L\sigma.
    \end{aligned}
\]
Taking $\widetilde{D}_{L, F} = D_{L, F} + 12L$, the result follows.
\end{proof}

\section{Proof of Theorem \ref{thm.1}}
\label{pot}

\subsection{Bounding the Decomposed Terms}

Before presenting the formal proof of Theorem~\ref{thm.1}, we state two key lemmas that control each term in the regret decomposition~\eqref{regret}.

\begin{lemma}
    \label{reb.1}
    For any $\delta > 0$, with probability at least $1 - \delta$,
    \[
        |I_1| \leq \sqrt{2H^2 K\eta}.
    \]
\end{lemma}

\begin{lemma}
    \label{reb.2}
    With probability at least
    \(
        1 - \big(3 + 8\,|\mathcal{C}_{\sigma^3/(D_{L_1,F,\sigma}\, KH^4)}|\big)\,H^2 K\,(\lfloor \log_2(HK)\rfloor + 3)\,\delta,
    \)
    the following bounds hold simultaneously:
    \[
        I_2 \leq \widetilde{O}\!\left(\sqrt{H^3 K} + \sqrt{H^2\,|\mathcal{C}_{\sigma}|\,|\widetilde{\mathcal{C}}_{\sigma}|} + \sqrt{HK\sigma}\right),
    \]
    \[
        I_3 \leq \widetilde{O}\!\left(\sqrt{H^3 K\,|\mathcal{C}_{\sigma}|} + H^2\,|\mathcal{C}_{\sigma}|\,|\widetilde{\mathcal{C}}_{\sigma}| + HK\sigma\right).
    \]
\end{lemma}

\subsection{Proof of Lemma \ref{reb.1}}
\begin{proof}[Proof of Lemma \ref{reb.1}]
 Let $\mathcal{F}^k = \sigma\big(\{(s_h^l, a_h^l, s_{h + 1}^l)\}_{1 \leq h \leq H, 1 \leq l \leq k - 1}\cup\{s_1^k\}\big)$, then $\big\{\sum_{h = 1}^Hr_h(s_h^k, a_h^k) - V_1^{\pi^k}(s_1^k)\big\}_k$ is a martingale difference sequence w.r.t. $\mathcal{F}^k$. Note that $0\leq \sum_{h = 1}^Hr_h(s_h^k, a_h^k)\leq H$. By Lemma \ref{aux.1} (Hoeffding's inequality for martingale), the event
 \[
    \widetilde{\mathcal{E}}_1 = \Bigg\{|I_1| = \bigg|\sum_{k = 1}^K\bigg(\sum_{h = 1}^Hr_h(s_h^k, a_h^k) - V_1^{\pi^k}(s_1^k)\bigg)\bigg| \leq \sqrt{2H^2K\eta}\Bigg\}
 \]
holds with probability at least $1 - \delta$.
\end{proof}

\subsection{Proof of Lemma \ref{reb.2}}
Our proof of Lemma \ref{reb.2} consists of several steps described below.
\bmhead{Step 1 : bounding $I_2$ and $I_3$} Now we seek to bound $I_2$ and $I_3$. To do so, it is useful to introduce  following concepts.
\[
\begin{aligned}
I_4 &:= \sum_{k = 1}^K\sum_{h = 1}^H\mathbb{V}({V}^k_{h + 1}, s_h^k, a_h^k);\\
I_5 &:= \sum_{k = 1}^K\sum_{h = 1}^H\mathbb{V}({V}^k_{h + 1} - V^*_{h + 1}, s_h^k, a_h^k).
\end{aligned}
\]
Now we show that $I_2, I_3$ can be dominated with high probability if we bound $I_4, I_5$.
\begin{lemma}
\label{reb2.1}
With probability at least $1 - (\lfloor\log_2(HK) + 3\rfloor)\delta$, the following event holds,
\[
\widetilde{\mathcal{E}}_2 = \Big\{|I_2| \leq 2\sqrt{I_4\eta} + 2H\eta\Big\}.
\]
\end{lemma}
\begin{proof}
For any $(h, k) \in [H] \times [K]$ and any integer $n = (k-1)H + h$, let $W_n = P_hV^k_{h + 1}(s_h^k, a_h^k) - V_{h + 1}^k(s_{h + 1}^k)$ and
\[
\mathcal{F}_n = \sigma\Big(\{s_{t}^{l}, a_{t}^{l}, s_{t + 1}^{l}\}_{1 \leq l < k, 1\leq t\leq H} \cup \{s^{k}_{t}, a^{k}_{t}, s_{t + 1}^{l}\}_{1\leq t < h} \cup \{s^{k}_{h}, a^{k}_{h}\}\Big).
\]
Then $\{W_n\}_n$ is a martingale difference sequence w.r.t. the natural filtration $\{\mathcal{F}_n\}_n$ . {By Lemma \ref{aux.2}} and the definition of $I_2$,
\[
|I_2| \leq 2\sqrt{\sum_{k = 1}^K\sum_{h = 1}^H\mathbb{V}({V}^k_{h + 1}, s_h^k, a_h^k)\eta} + 2H\eta.
\]
holds with probability at least $1 - (\lfloor\log_2(HK)\rfloor+3)\delta$.
\end{proof}
For any $1 < k \leq K, h \in [H]$, define
\[
(\Tilde{s}_h^k, \Tilde{a}^k_h) = {\arg\min}_{l \in [k-1]}\rho\big[(s_h^k, a_h^k), (s_h^l, a_h^l)\big].
\]
Recall the definition of events $\mathcal{E}_{L, F}$ (Lemma \ref{concentration}) and $\mathcal{E}_{f}$ (Lemma \ref{con.1}). Let $\mathcal{E}_1$ be the event $\mathcal{E}_{L, F}$ with $L = 2L_1, F = H$ and $\mathcal{E}_2$ be the event $\mathcal{E}_{L, F}$ with $L = 4L_1H, F = H^2$. Moreover, define
\[
\begin{aligned}
\mathcal{E}_3 &= \cap_{h = 1}^{H}\mathcal{E}_{V^*_{h + 1}}\\
\mathcal{E}_4 &= \cap_{h = 1}^{H}\mathcal{E}_{(V^*_{h + 1})^2},
\end{aligned}
\]
and let
\[
\widetilde{\mathcal{E}}_3 = \mathcal{E}_1\cap\mathcal{E}_2\cap\mathcal{E}_3\cap\mathcal{E}_4.
\]
The following lemma gives an upper bound of $I_3$.
\begin{lemma}
\label{reb2.2}
In the event $\widetilde{\mathcal{E}}_3$, we have
\[
I_3 \leq \hat{I}_3 \leq I_{31} + I_{32} + I_{33} + M_{\sigma
	, \delta}HK\sigma,
\]
where
\[
\begin{aligned}
    \hat{I}_3&=\sum_{k = 1}^K\sum_{h = 1}^H\mathbf{1}\{\mathcal{E}_h^k\}\big(b^k_h(\Tilde{s}^k_h, \Tilde{a}^k_h) + \big(\widehat{P}^k_h - P_h\big){V}^k_{h + 1}(\Tilde{s}^k_h, \Tilde{a}^k_h) + 2(L_1 + \lambda_r + \lambda_p L_1)\sigma\big) + H^2|\mathcal{C}_{\sigma}|\\
	I_{31} &=  44\sum_{k = 1}^K\sum_{h = 1}^H\mathbf{1}\{\mathcal{E}_h^k\}\sqrt{\frac{\V(V^k_{h + 1} - V^*_{h + 1}, \Tilde{s}_h^k, \Tilde{a}^k_h)D_{L_1,H,\sigma}}{C^k_h(\Tilde{s}_h^k, \Tilde{a}^k_h)}};\\
	I_{32} &= 35\sum_{k = 1}^K\sum_{h = 1}^H\mathbf{1}\{\mathcal{E}_h^k\}\sqrt{\frac{\V(V^k_{h +1}, \Tilde{s}_h^k, \Tilde{a}^k_h)\eta_4}{C^k_h(\Tilde{s}_h^k, \Tilde{a}^k_h)}};\\
	I_{33} &= H^2|\mathcal{C}_{\sigma}| +  \sum_{k = 1}^K\sum_{h = 1}^H\mathbf{1}\{\mathcal{E}_h^k\}\frac{HF_1(\eta, H, \sigma)}{C^k_h(\Tilde{s}_h^k, \Tilde{a}^k_h)};\\
	M_{\sigma, \delta} &= 2L_1 + 2\lambda_r + 2\lambda_p L_1 + D_{\delta}  + 6\widetilde{D}_{L_1, H};\\
	F_1(\eta, H, \sigma) &= 12D_{L_1,H,\sigma} + 165\eta_4,
\end{aligned}
\]
events $\mathcal{E}_h^k = \{\rho[(s_h^k, a_h^k), (\widetilde{s}_h^k, \widetilde{a}_h^k)] \leq 2\sigma\}$ and $|\mathcal{C}_{\sigma}|$ are the $\sigma$-covering number of $\mathcal{S} \times \mathcal{A}$.
\end{lemma}
\begin{proof}
Let $\square_h^k := \rho\big[(s_h^k, a_h^k), (\Tilde{s}_h^k, \Tilde{a}^k_h)\big]$, it is easy to see that
\beqnn
I_3 &=& \sum_{k = 1}^K\sum_{h = 1}^H(V_h^k(s_h^k) - r_h(s_h^k, a_h^k) - P{V}^k_{h + 1}(s_h^k, a_h^k))\\
&\leq &\sum_{k = 1}^K\sum_{h = 1}^H\mathbf{1}\{\mathcal{E}_h^k\}\big({V}_h^k(s_h^k) - r_h(s^k_h, a^k_h) - P{V}^k_{h + 1}(s_h^k, a_h^k)\big) + H\sum_{k = 1}^K\sum_{h = 1}^H\mathbf{1}\{\bar{\mathcal{E}}_h^k\}.
\eeqnn
By the definition of $V_h^k$, $Q_{h}^k$ and $\tilde{Q}_h^k$, we have that
$$V_h^k(s_h^k) \leq Q_h^k(s_h^k, a_h^k) \leq \Tilde{Q}_h^k(\tilde{s}_h^k, \tilde{a}^k_h) + L_h\square_h^k.$$
In addition, when the event $\bar{\mathcal{E}}_h^k$ happens, $(s_h^k, a_h^k)$ has to be in a $\sigma$-ball which does not contain the state-action pair in the past data. Therefore,
\begin{align}
	I_3 &\leq \sum_{k = 1}^K\sum_{h = 1}^H\mathbf{1}\{\mathcal{E}_h^k\}\big(\Tilde{Q}_h^k(\tilde{s}_h^k, \tilde{a}^k_h) + L_h\square_h^k  - r_h(s^k_h, a^k_h) - P{V}^k_{h + 1}(s_h^k, a_h^k)\big) + H^2|\mathcal{C}_{\sigma}|\notag\\
	& \leq \sum_{k = 1}^K\sum_{h = 1}^H\mathbf{1}\{\mathcal{E}_h^k\}\big(\Tilde{Q}_h^k(\Tilde{s}^k_h, \Tilde{a}^k_h) - r_h(\Tilde{s}^k_h, \Tilde{a}^k_h) - P_h{V}^k_{h + 1}(\Tilde{s}^k_h, \Tilde{a}^k_h) + (L_h  + \lambda_r + \lambda_p L_{h + 1})\square_h^k\big) + H^2|\mathcal{C}_{\sigma}|\notag\\
	& \leq \sum_{k = 1}^K\sum_{h = 1}^H\mathbf{1}\{\mathcal{E}_h^k\}\big(b^k_h(\Tilde{s}^k_h, \Tilde{a}^k_h) + \big|\big(\widehat{P}^k_h - P_h\big){V}^k_{h + 1}(\Tilde{s}^k_h, \Tilde{a}^k_h)\big| + 2(L_1 + \lambda_r + \lambda_p L_1)\sigma\big) + H^2|\mathcal{C}_{\sigma}| =\hat{I}_3,\label{app.reb.1}
\end{align}
where the second inequality is by the definition of $(\Tilde{s}^k_h, \Tilde{a}^k_h)$ and the last inequality is based on the definition of $\mathcal{E}_h^k$ and the fact $L_h \leq L_1$.
Recall that
\begin{equation}
	\label{app.reb.2}
	b_h^k(s, a) = 9\sqrt{\frac{\widehat{\V}^k_h(V_{h + 1}^k, s, a)\eta_4}{C^k_h(s, a)}} + \frac{162H\eta_4}{C^k_h(s, a)} + \sigma D_{\delta}.
\end{equation}
Note that $V^*_{h + 1}$ is $L_1$-Lipschitz,  $|V^*_{h + 1}| \leq H$ and $V_{h + 1}^k(s) - V^*_{h +1}(s) \in \mathcal{F}_{2L_1, H}$, Since $\mathcal{E}_1$ and $\mathcal{E}_3$ hold, for any $(s, a, h, k) \in \mathcal{S} \times \mathcal{A} \times [H] \times[K]$, we have
\begin{align}
	&\big|\big(\widehat{P}^k_h - P_h\big)V_{h + 1}^k(s,a)\big| \leq \big|\big(\widehat{P}^k_h - P\big)V_{h + 1}^*(s,a)\big| + \big|\big(\widehat{P}^k_h - P\big)(V^k_{h + 1} - V^*_{h + 1})(s, a)\big|\notag\\
	&\leq 2\sqrt{\frac{\V_h(V^*_{h + 1}, s, a)\eta_1}{C^k_h(s, a)}} + \frac{H\eta_5}{C^k_h(s, a)} + \sigma{D}_{L_1, H} +2\sqrt{\frac{\V(V^k_{h + 1} - V^*_{h +1}, s, a)D_{2L_1, H, \sigma}}{C^k_h(s, a)}}\notag\\
	& + \frac{2HD_{2L_1, H, \sigma}}{C^k_h(s, a)} + \sigma\widetilde{D}_{2L_1, H}\notag\\
	&\leq 2\sqrt{\frac{\V_h(V^*_{h + 1}, s, a)\eta_1}{C^k_h(s, a)}} + 2\sqrt{\frac{\V_h(V^k_{h + 1} - V^*_{h +1}, s, a)D_{2L_1, H, \sigma}}{C^k_h(s, a)}} + \frac{3HD_{2L_1, H, \sigma}}{C^k_h(s, a)} + 3\sigma\widetilde{D}_{L_1, H} \label{app.reb.3}\\
	&\leq \frac{1}{H}\V_h(V^*_{h + 1}, s, a) + \frac{1}{H}\V_h(V^k_{h + 1} - V^*_{h +1}, s, a) + \frac{5HD_{2L_1, H, \sigma}}{C^k_h(s, a)} + 3\sigma\widetilde{D}_{L_1, H}\label{app.reb.3.1}
\end{align}
holds for all $(s, a) \in \mathcal{S} \times \mathcal{A}$ and all $(k, h) \in [K] \times [H]$, where the second inequality holds since $D_{2L_1, H, \sigma} \geq \eta_5 \geq \eta_1$ and $2\widetilde{D}_{L_1, H} \geq \widetilde{D}_{2L_1, H}$. In the last inequality, we use the fact $2\sqrt{ab} \leq a + b$ for $a, b \geq 0$. From equations (\ref{app.reb.1}), (\ref{app.reb.2}) and the inequality (\ref{app.reb.3}), we have
\begin{align}
	I_3\leq &H^2|\mathcal{C}_{\sigma}|+\sum_{k = 1}^K\sum_{h = 1}^H\mathbf{1}\{\mathcal{E}_h^k\}\Bigg(9\sqrt{\frac{\widehat{\V}^k_h(V_{h + 1}^k, \Tilde{s}^k_h, \Tilde{a}^k_h)\eta_4}{C^k_h(\Tilde{s}^k_h, \Tilde{a}^k_h)}}  + 2\sqrt{\frac{\V_h(V^k_{h + 1} - V^*_{h +1}, \Tilde{s}^k_h, \Tilde{a}^k_h)D_{2L_1, H, \sigma}}{C^k_h(\Tilde{s}^k_h, \Tilde{a}^k_h)}}\notag\\
	& + 2\sqrt{\frac{\V_h(V^*_{h + 1}, \Tilde{s}^k_h, \Tilde{a}^k_h)\eta_1}{C^k_h(\Tilde{s}^k_h, \Tilde{a}^k_h)}}+\frac{(3D_{2L_1, H, \sigma} + 162\eta_4)H}{C^k_h(\Tilde{s}^k_h, \Tilde{a}^k_h)}
	+ \sigma\big(2L_1 + 2\lambda_r + 2\lambda_p L_1 + D_{\delta} + 3\widetilde{D}_{L_1, H}\big)\Bigg)\\
	\leq &H^2|\mathcal{C}_{\sigma}|+\sum_{k = 1}^K\sum_{h = 1}^H\mathbf{1}\{\mathcal{E}_h^k\}\Bigg(9\sqrt{\frac{\widehat{\V}^k_h(V_{h + 1}^k, \Tilde{s}^k_h, \Tilde{a}^k_h)\eta_4}{C^k_h(\Tilde{s}^k_h, \Tilde{a}^k_h)}}  + 5\sqrt{\frac{\V_h(V^k_{h + 1} - V^*_{h +1}, \Tilde{s}^k_h, \Tilde{a}^k_h)D_{2L_1, H, \sigma}}{C^k_h(\Tilde{s}^k_h, \Tilde{a}^k_h)}}\notag\\
	& + 3\sqrt{\frac{\V_h(V^k_{h + 1}, \Tilde{s}^k_h, \Tilde{a}^k_h)\eta_1}{C^k_h(\Tilde{s}^k_h, \Tilde{a}^k_h)}}+\frac{(3D_{2L_1, H, \sigma} + 162\eta_4)H}{C^k_h(\Tilde{s}^k_h, \Tilde{a}^k_h)}
	+ \sigma\big(2L_1 + 2\lambda_r + 2\lambda_p L_1 + D_{\delta} + 3\widetilde{D}_{L_1, H}\big)\Bigg)
	 \label{app.reb.4}
\end{align}
where for any $(s, a) \in \mathcal{S} \times \mathcal{A}$, it is easy to see that
\begin{align}
	\widehat{\V}^k_h(V_{h + 1}^k, s, a) &= \widehat{\V}(V_{h + 1}^k, s, a) - \V^k(V_{h + 1}^k, s, a) + \V(V_{h + 1}^k , s, a)\notag\\
	&\leq
	\big|\big(\widehat{P}^k_h - {P}_h\big)(V_{h + 1}^k)^2(s, a)\big| + 2H\big|(\widehat{P}^k_h - P)V_{h + 1}^k(s, a)\big| + {\V}(V_{h + 1}^k , s, a).\label{app.reb.5}
\end{align}
For the first term, in the events $\mathcal{E}_2$ and $\mathcal{E}_4$,
\begin{align}
	\label{app.reb.6}
	&\big|\big(\widehat{P}^k_h - {P}_h\big)(V_{h + 1}^k)^2(s, a)\big| \leq 2\big|\big(\widehat{P}^k_h - {P}_h\big)(V_{h + 1}^k - V^*_{h + 1})^2(s, a)\big| + 2\big|\big(\widehat{P}^k_h - {P}_h\big)(V^*_{h + 1})^2(s, a)\big|\notag\\
	&\qquad\leq 4\sqrt{\frac{\V_h((V_{h + 1}^k - V^*_{h + 1})^2, s, a)D_{4L_1H, H^2, \sigma}}{C^k_h(s, a)}} + \frac{4H^2D_{4L_1H, H^2, \sigma}}{C^k_h(s, a)}\notag\\
	&\qquad\quad+ {2\sigma\widetilde{D}_{4L_1H, H^2}}
	+ 4\sqrt{\frac{\V((V^*_{h})^2, s, a)\eta_1}{C^k_h(s, a)}}+\frac{2H\eta_5}{C^k_h(s, a)} + { 2\sigma{D}_{2L_1H, H^2}},\notag\\
	&\qquad\leq 8\sqrt{\frac{\V(V_{h + 1}^k - V^*_{h + 1}, s, a)H^2D_{4L_1H, H^2, \sigma}}{C^k_h(s, a)}} + \frac{4H^2D_{4L_1H, H^2, \sigma}}{C^k_h(s, a)}\notag\\
	&\qquad\quad+ 2\sigma\widetilde{D}_{4L_1H, H^2}
	+ 8\sqrt{\frac{\V(V^*_{h}, s, a)H^2\eta_1}{C^k_h(s, a)}} + \frac{2H\eta_5}{C^k_h(s, a)} + 2\sigma{D}_{2L_1H, H^2},\notag\\
	&\qquad\leq 4\V(V_{h + 1}^k - V^*_{h + 1}, s, a) + 4\V(V^*_{h}, s, a) + \frac{14H^2D_{L_1,H,\sigma}}{C^k_h(s, a)} + { 12\sigma\widetilde{D}_{L_1H, H^2}},
\end{align}
where the second inequality employs Lemmas \ref{concentration} and \ref{con.1}, the third applies Lemma \ref{aux.7}, while the final inequality directly follows from inequalities
$\widetilde{D}_{4L_1H, H^2} \leq 2\widetilde{D}_{2L_1H, H^2} \leq 4\widetilde{D}_{L_1H, H^2}\;\text{and}\;            \eta_5\leq D_{4L_1H, H^2, \sigma} \leq D_{L_1,H,\sigma}.$ Note that $\widetilde{D}_{L_1H, H}\leq H\widetilde{D}_{L_1H, H}$, invoking (\ref{app.reb.3.1}) and (\ref{app.reb.6}) to (\ref{app.reb.5}), we obtain
\begin{align}
	\label{app.reb.7}
	\widehat{\V}^k_h(V_{h + 1}^k, s, a)&\leq 6\V(V_{h + 1}^k - V^*_{h + 1}, s, a) + 6\V(V^*_{h}, s, a) + \V(V^k_{h}, s, a) +  \frac{24H^2D_{L_1,H,\sigma}}{C^k_h(s, a)}+ { 18\sigma H\widetilde{D}_{L_1, H}} \notag\\
	&\leq 18\V(V_{h + 1}^k - V^*_{h + 1}, s, a) + 13\V(V^k_{h}, s, a) +  \frac{24H^2D_{L_1,H,\sigma}}{C^k_h(s, a)}+ {18\sigma H\widetilde{D}_{L_1, H}},
\end{align}
where the final line is by the fact $\V(f + g, s, a) \leq 2\V(f, s, a) + 2\V(g, s, a)$ for any $s, a, f, g$.
Combining (\ref{app.reb.4}) and (\ref{app.reb.7}) leads to
\begin{align}
	\label{app.reb.8}
	I_3 &\leq  H^2|\mathcal{C}_{\sigma}| + \sum_{k = 1}^K\sum_{h = 1}^H\mathbf{1}\{\mathcal{E}_h^k\}\Bigg(44\sqrt{\frac{\V(V^k_{h + 1} - V^*_{h + 1}, \Tilde{s}_h^k, \Tilde{a}^k_h)D_{L_1,H,\sigma}}{C^k_h(\Tilde{s}_h^k, \Tilde{a}^k_h)}} + 35\sqrt{\frac{\V(V^k_{h + 1}, \Tilde{s}_h^k, \Tilde{a}^k_h)\eta_4}{C^k_h(\Tilde{s}_h^k, \Tilde{a}^k_h)}}\notag\\
	&+ \frac{\sqrt{24H^2\eta_{4}D_{L_1,H,\sigma}}}{C^k_h(\Tilde{s}_h^k, \Tilde{a}^k_h)} + \sqrt{\frac{18\sigma H\eta_{4}\widetilde{D}_{L_1, H}}{C^k_h(\Tilde{s}_h^k, \Tilde{a}^k_h)}} + \frac{(3D_{2L_1, H, \sigma} + 162\eta_4)H}{C^k_h(\Tilde{s}_h^k, \Tilde{a}^k_h)} \notag\\
	&  + \sigma\big(2L_1 + 2\lambda_r + 2\lambda_p L_1 + D_{\delta}  + 3\widetilde{D}_{L_1, H}\big)\Bigg)\notag\\
	&\leq  H^2|\mathcal{C}_{\sigma}| + \sum_{k = 1}^K\sum_{h = 1}^H\mathbf{1}\{\mathcal{E}_h^k\}\Bigg(44\sqrt{\frac{\V(V^k_{h + 1} - V^*_{h + 1}, \Tilde{s}_h^k, \Tilde{a}^k_h)D_{L_1,H,\sigma}}{C^k_h(\Tilde{s}_h^k, \Tilde{a}^k_h)}} + 35\sqrt{\frac{\V(V^k_{h + 1}, \Tilde{s}_h^k, \Tilde{a}^k_h)\eta_4}{C^k_h(\Tilde{s}_h^k, \Tilde{a}^k_h)}}\notag\\
	& + \frac{(12D_{L_1,H,\sigma} + 165\eta_4)H}{C^k_h(\Tilde{s}_h^k, \Tilde{a}^k_h)} + \sigma\big(2L_1 + 2\lambda_r + 2\lambda_p L_1 + D_{\delta}  + 6\widetilde{D}_{L_1, H}\big)\Bigg),
\end{align}
where the inequality is based on $9\sqrt{18} + 5 < 44, 9\sqrt{13} + 3 < 35$. Letting $F_1(\eta, H, \sigma) = 12D_{L_1,H,\sigma} + 165\eta_4$ implies the result.
\end{proof}
\begin{lemma}
\label{reb2.3}
Let $\beta = \{\beta_h^k \geq 0 : 1 \leq h \leq H, 1 \leq k \leq K\}$ be a sequence of non-negative numbers, let $\mathcal{C}_{\sigma} = \{(s_j, a_j) \in \mathcal{S} \times \mathcal{A}, j = 1, 2, \cdots, |\mathcal{C}_{\sigma}|\}$ be a $\sigma$-covering of $\mathcal{S} \times \mathcal{A}$, then the following inequalities hold.
\begin{align}
	\sum_{k = 1}^K\sum_{h = 1}^H\sqrt{\frac{\beta_h^k\mathbf{1}\{\mathcal{E}_h^k\}}{C^k_h(\Tilde{s}_h^k, \Tilde{a}_h^k)}} &\leq\sqrt{C_3H|\mathcal{C}_{\sigma}|\sum_{k = 1}^K\sum_{h = 1}^H\beta_h^k\mathbf{1}\{\mathcal{E}_h^k\}}, \label{reb2.3.1}\\
	\sum_{k = 1}^K\sum_{h = 1}^H\frac{\mathbf{1}\{\mathcal{E}_h^k\}}{C^k_h(\Tilde{s}_h^k, \Tilde{a}_h^k)} &\leq C_3H|\mathcal{C}_{\sigma}|\label{reb2.3.2}.
\end{align}
where $C_3=\frac{1}{\beta}+\frac{1}{g(4)}\log(1+\frac{g(4)K}{\beta})$.
\end{lemma}
\begin{proof}
Define a partition $\{B_j\}_{j = 1}^{|\mathcal{C}_{\sigma}|}$ of $\mathcal{S} \times \mathcal{A}$ as follows:
\[
B_j = \bigg\{(s, a) \in \mathcal{S} \times \mathcal{A} : (s_j, a_j) = {\arg\min}_{(s_i, a_i)\in \mathcal{C}_{\sigma}}\rho[(s, a), (s_i, a_i)]\bigg\}
\]
Define the number of visits to each set $B_j$ as $$N^k_h(B_j) = \sum_{l = 1}^{k - 1}\sum_{h = 1}^H\mathbf{1}\{(s_h^l, a_h^l) \in B_j\}.$$ If $\mathcal{E}_h^k$ holds, assume that $(s_h^k, a_h^k) \in B_j$, then we have $\rho\big[(s_h^k, a_h^k), (s_{h}^l, a_{h}^l)\big] \leq 4\sigma$ for each $(s_{h}^l, a_{h}^l) \in B_j$. It follows that
\begin{align}
	C^k_h(\Tilde{s}_h^k, \Tilde{a}_h^k)= &\beta + \sum_{l = 1}^{k - 1}g\bigg(\frac{\rho[(\Tilde{s}_h^k, \Tilde{a}_h^k), (s_{h}^l, a^l_{h})]}{\sigma}\bigg)\notag\\
	\geq&\beta + \sum_{l = 1}^{k - 1}g\bigg(\frac{\rho[(\Tilde{s}_h^k, \Tilde{a}_h^k), (s_{h}^l, a^l_{h})]}{\sigma}\bigg)\mathbf{1}\{(s_{h}^l, a_{h}^l) \in B_j\}\notag\\
	\geq &\beta + g(4)\sum_{l = 1}^{k - 1}\mathbf{1}\{(s_{h}^l, a_{h}^l) \in B_j\}\notag\\
	= &\beta(1 + g(4)\beta^{-1}N^k_h(B_j)),\label{reb.5.1}
\end{align}
where we use the assumption that $g$ is not increasing.
For convenience, denote $\mathbf{1}\{(s_{h}^k, a_{h}^k) \in B_j\}$ by $\mathbf{1}(h, k, j)$. By Cauchy inequality,
\begin{align}
	\label{reb.5.cauchy}
	\sum_{k = 1}^K\sum_{h = 1}^H\sqrt{\frac{\beta_h^k\mathbf{1}\{\mathcal{E}_h^k\}}{C^k_h(\Tilde{s}_h^k, \Tilde{a}_h^k)}}
	\leq \sqrt{\sum_{k = 1}^K\sum_{h = 1}^H\beta_h^k\mathbf{1}\{\mathcal{E}_h^k\}}\sqrt{\sum_{j = 1}^{|\mathcal{C}_{\sigma}|}\sum_{k = 1}^K\sum_{h = 1}^H\frac{\mathbf{1}\{\mathcal{E}_h^k\}\mathbf{1}(h, k, j)}{C^k_h(\Tilde{s}_h^k, \Tilde{a}_h^k)}}.
\end{align}
Furthermore, from (\ref{reb.5.1}) and Lemma \ref{aux.num} it follows that
\[ \begin{aligned}
	\sqrt{\sum_{j = 1}^{|\mathcal{C}_{\sigma}|}\sum_{k = 1}^K\sum_{h = 1}^H\frac{\mathbf{1}\{\mathcal{E}_h^k\}\mathbf{1}(h, k, j)}{C^k_h(\Tilde{s}_h^k, \Tilde{a}_h^k)}}&\leq \beta^{-1/2}\sqrt{\sum_{j = 1}^{|\mathcal{C}_{\sigma}|}\sum_{k = 1}^K\sum_{h = 1}^H\frac{\mathbf{1}(h, k, j)}{1 + g(4)\beta^{-1}N_h^k(B_j)}}\\
	&\leq \beta^{-1/2}\sqrt{H|\mathcal{C}_{\sigma}| + \sum_{j = 1}^{|\mathcal{C}_{\sigma}|}\sum_{h = 1}^H\int_{0}^{N^K_h(B_j)}\frac{dz}{1 + g(4)\beta^{-1}z}}
\end{aligned}
\]
which further implies that
\[
\begin{aligned}
	\sqrt{\sum_{j = 1}^{|\mathcal{C}_{\sigma}|}\sum_{k = 1}^K\sum_{h = 1}^H\frac{\mathbf{1}\{\mathcal{E}_h^k\}\mathbf{1}(h, k, j)}{C^k_h(\Tilde{s}_h^k, \Tilde{a}_h^k)}}&\leq \beta^{-1/2}\sqrt{H|\mathcal{C}_{\sigma}| + \frac{\beta}{g(4)}\sum_{j = 1}^{|\mathcal{C}_{\sigma}|}\sum_{h = 1}^H\log(1 + g(4)\beta^{-1}N_h^K(B_j))}\\
	&\leq \sqrt{H|\mathcal{C}_{\sigma}|\Big(\frac{1}{\beta}+\frac{1}{g(4)}\log(1+\frac{g(4)K}{\beta})\Big)},
\end{aligned}
\]
Similarly, noting that
\[
\sum_{k = 1}^K\sum_{h = 1}^H\frac{\mathbf{1}\{\mathcal{E}_h^k\}}{C^k_h(\Tilde{s}_h^k, \Tilde{a}_h^k)} \leq \sum_{j = 1}^{|\mathcal{C}_{\sigma}|}\sum_{k = 1}^K\sum_{h = 1}^H\frac{\mathbf{1}\{\mathcal{E}_h^k\}\mathbf{1}(h, k, j)}{C^k_h(\Tilde{s}_h^k, \Tilde{a}_h^k)},
\]
we also have (\ref{reb2.3.2}).
\end{proof}
\begin{lemma}
\label{reb2.4}
In the event $\widetilde{\mathcal{E}}_3$, we have
\[
\begin{aligned}
	I_3 \leq \hat{I}_3 \leq 44\sqrt{I_5C_3|\mathcal{C}_{\sigma}|HD_{L_1,F,\sigma}} + 35\sqrt{I_4C_3|\mathcal{C}_{\sigma}|H\eta_{4}} + 232 HK\sigma + H^2|\mathcal{C}_{\sigma}|F_2(\eta, H, \sigma)+M_{\sigma
		, \delta}HK\sigma,
\end{aligned}
\]
where
\[
F_2(\eta, H, \sigma) =1+ C_3F_1(\eta, H, \sigma) + 44C_3\lambda_pL_1D_{L_1,H,\sigma} + 35C_3\lambda_pL_1\eta_4.
\]
\end{lemma}
\begin{proof}
Following the same discussion in (\ref{app.con.3}), we have
\begin{align}
	\sum_{k = 1}^K\sum_{h = 1}^H\mathbf{1}\{\mathcal{E}_h^k\}\V
	_h(V_h^k, \tilde{s}_h^k, \tilde{a}_h^k) &\leq \sum_{k = 1}^K\sum_{h = 1}^H\V_h(V_h^k, s_h^k, a_h^k) + 8\lambda_pL_1H^2K\sigma,\label{app.reb.9.1}\\
	\sum_{k = 1}^K\sum_{h = 1}^H\mathbf{1}\{\mathcal{E}_h^k\}\V_h(V_{h + 1}^k - V^*_{h + 1}, \tilde{s}_h^k, \tilde{a}_h^k) &\leq \sum_{k = 1}^K\sum_{h = 1}^H\V_h(V_{h + 1}^k - V^*_{h + 1}, s_h^k, a_h^k) + 16\lambda_pL_1H^2K\sigma.\label{app.reb.9.2}
\end{align}
Lemma \ref{reb2.3} implies that
\[
\begin{aligned}
	I_{31} &\leq 44\sqrt{\sum_{k = 1}^K\sum_{h = 1}^H\mathbf{1}\{\mathcal{E}_h^k\}\V_h(V^k_{h + 1} - V^*_{h + 1}, \Tilde{s}_h^k, \Tilde{a}^k_h)C_3H|\mathcal{C}_{\sigma}|D_{L_1,H,\sigma}} ;\\
	I_{32} &\leq 35\sqrt{\sum_{k = 1}^K\sum_{h = 1}^H\mathbf{1}\{\mathcal{E}_h^k\}\V_h(V^k_{h +1}, \Tilde{s}_h^k, \Tilde{a}^k_h)C_3H|\mathcal{C}_{\sigma}|\eta_{4}};\\
	I_{33} &\leq H^2|\mathcal{C}_{\sigma}| +  C_3H^2|\mathcal{C}_{\sigma}|F_1(\eta, H, \sigma).
\end{aligned}
\]
Combining the above inequalities and applying Cauchy's inequality, we get the desired result from Lemma \ref{reb2.2}.
\end{proof}

\bmhead{Step 2 : bounding $I_4, I_5$} Now we seek to bound $I_4$ and $I_5$.
\begin{lemma}
\label{reb2.5}
$\pr(\widetilde{\mathcal{E}}_4) \geq 1 - 2(\lfloor\log_{2}(HK)\rfloor + 3)\delta$ where
$$
\widetilde{\mathcal{E}}_4 = \big\{I_4 \leq 20H^2\eta + 4H\max\{I_3, 0\} + 2H^2K\big\}\cap\big\{I_5 \leq 20H^2\eta + 4H\max\{I_3, 0\}\big\}.
$$

\end{lemma}
\begin{proof}
It is easy to see that
\begin{align}
	&I_4 = \sum_{k = 1}^K\sum_{h = 1}^H\V\big(V^k_{h + 1}, s_h^k, a_h^k\big)\notag\\
	&= \sum_{k = 1}^K\sum_{h = 1}^H\Big[\Big(P\big[V^k_{h +1}\big]^2(s_h^k, a_h^k) - \big[V_{h+1}^k(s_{h+1}^k)\big]^2\Big) + \Big(\big[V_{h}^k(s_{h}^k)\big]^2 - \big[PV^k_{h + 1}\big]^2(s_h^k, a_h^k)\Big)\Big] - \sum_{k = 1}^K\big[V_{1}^k(s_{1}^k)\big]^2\notag\\
	&\leq \sum_{k = 1}^{K}\sum_{h = 1}^{H}\Big(P[V^k_{h + 1}]^2(s_h^k, a_h^k) - \big[V^k_{h + 1}(s_{h + 1}^k)\big]^2\Big) + 2H\sum_{k = 1}^{K}\sum_{h = 1}^{H}\max\left\{\Big(V^k_{h}(s_{h}^k) - PV^k_{h + 1}(s_h^k, a_h^k)\Big), 0\right\}.\label{i4_upper_bound}
\end{align}

Similar to the proof of Lemma \ref{reb2.1}, with probability at least $1 - (\lfloor\log_{2}(HK)\rfloor + 3)\delta$,

\begin{align}
	\sum_{k = 1}^{K}\sum_{h = 1}^{H}\Big(P[V^k_{h + 1}]^2(s_h^k, a_h^k) - \big[V^k_{h + 1}(s_{h + 1}^k)\big]^2\Big) &\leq 2\sqrt{\sum_{k = 1}^{K}\sum_{h = 1}^{H}\V((V_{h + 1}^k)^2, s_h^k, a_h^k)\eta} + 2H^2\eta\notag\\
	&\leq 4H\sqrt{\sum_{k = 1}^{K}\sum_{h = 1}^{H}\V(V_{h + 1}^k, s_h^k, a_h^k)\eta} + 2H^2\eta,\label{i4_first_1upper_bound}
\end{align}

where the final inequality is based on Lemma \ref{aux.7}.
By the definition of $\hat{I}_3$, we have that
\begin{align}
    &\sum_{k = 1}^{K}\sum_{h = 1}^{H}\max\left\{\Big(V^k_{h}(s_{h}^k) - PV^k_{h + 1}(s_h^k, a_h^k)\Big), 0\right\}\notag\\
    \leq &\sum_{k = 1}^{K}\sum_{h = 1}^{H}\mathbf{1}\{\mathcal{E}_h^k\}\max\left\{\Big(Q^k_{h}(s_{h}^k, a_h^k) - PV^k_{h + 1}(s_h^k, a_h^k)\Big), 0\right\} + \sum_{k = 1}^{K}\sum_{h = 1}^{H}\mathbf{1}\{\bar{\mathcal{E}}_h^k\}\notag\\
    \leq &\sum_{k = 1}^{K}\sum_{h = 1}^{H}\max\left\{\Big(\widetilde{Q}^k_{h}(\tilde{s}_{h}^k, \tilde{a}_h^k) - PV^k_{h + 1}(\tilde{s}_h^k, \tilde{a}_h^k) + L_h\square_h^k\Big), 0\right\} + H|\mathcal{C}_{\sigma}|\notag\\
    \leq &\sum_{k = 1}^{K}\sum_{h = 1}^{H}\mathbf{1}\{\mathcal{E}_h^k\}\max\left\{\Big(r^k_{h}(\tilde{s}_{h}^k, \tilde{a}_h^k) + b_h^k(\tilde{s}_{h}^k, \tilde{a}_h^k) + |(\widehat{P}^k_h - P)V^k_{h + 1}(\tilde{s}_h^k, \tilde{a}_h^k)| + L_h\sigma\Big), 0\right\} + H|\mathcal{C}_{\sigma}|\notag\\
    \leq &\sum_{k = 1}^{K}\sum_{h = 1}^{H}\mathbf{1}\{\mathcal{E}_h^k\}\max\left\{\Big(b_h^k(\tilde{s}_{h}^k, \tilde{a}_h^k) + |(\widehat{P}^k_h - P)V^k_{h + 1}(\tilde{s}_h^k, \tilde{a}_h^k)| + (L_r + L_h)\sigma\Big), 0\right\} + H|\mathcal{C}_{\sigma}| + \sum_{k = 1}^{K}\sum_{h = 1}^{H}r_h(s_h^k,a_h^k)\notag\\
    \leq &\hat{I}_3 + HK \label{i4_second_upper_bound}
\end{align}
Combining \eqref{i4_upper_bound}, \eqref{i4_first_1upper_bound} and \eqref{i4_second_upper_bound} yields that
\[
\begin{aligned}
	I_4 &\leq \sum_{k = 1}^{K}\sum_{h = 1}^{H}\Big(P[V^k_{h + 1}]^2(s_h^k, a_h^k) - \big[V^k_{h + 1}(s_{h + 1}^k)\big]^2\Big) + 2H\hat{I}_3 + 2H^2K\\
	&\leq 4H\sqrt{I_4\eta} + 2H^2\eta + 2H\hat{I}_3 + 2H^2K.
\end{aligned}
\]
It follows that
\[
\big(\sqrt{I_4} - 2H\sqrt{\eta}\big)^2 \leq 6H^2\eta + 2H\hat{I}_3 + 2H^2K,
\]
which implies
\begin{align}
	I_4 \leq 20H^2\eta + 4H\hat{I}_3 + 4H^2K\label{app.reb.10}.
\end{align}
Let $\check{V}_h^k = V^k_h - V^*_{h}$. Similarly, we have
\[
\begin{aligned}
	I_5 &= \sum_{k = 1}^K\sum_{h = 1}^H\V\big(\check{V}^k_{h + 1}, s_h^k, a_h^k\big)\\
	&= \sum_{k = 1}^K\sum_{h = 1}^H\Big[\Big(P\big[\check{V}^k_{h +1}\big]^2(s_h^k, a_h^k) - \big[\check{V}_{h +1}^k(s_{h + 1}^k)\big]^2\Big) + \Big(\big[\check{V}_{h}^k(s_{h}^k)\big]^2 - \big[P\check{V}^k_{h + 1}\big]^2(s_h^k, a_h^k)\Big)\Big] - \sum_{k=1}^K\big[\check{V}_{1}^k(s_{1}^k)\big]^2\\
	&\leq \sum_{k = 1}^{K}\sum_{h = 1}^{H}\Big(P[\check{V}^k_{h + 1}]^2(s_h^k, a_h^k) - \big[\check{V}^k_{h + 1}(s_{h + 1}^k)\big]^2\Big)+ 2H\sum_{k = 1}^{K}\sum_{h = 1}^{H}\max\left\{\Big(\check{V}^k_{h + 1}(s_{h + 1}^k) - P\check{V}^k_{h + 1}(s_h^k, a_h^k)\Big), 0\right\}.
\end{aligned}
\]
With probability at least $1 - (\lfloor\log_{2}(HK)\rfloor + 3)\delta$,
\beqnn
\sum_{k = 1}^{K}\sum_{h = 1}^{H}\Big(P[\check{V}^k_{h + 1}]^2(s_h^k, a_h^k) - \big[\check{V}^k_{h + 1}(s_{h + 1}^k)\big]^2\Big) &\leq &2\sqrt{\sum_{k = 1}^{K}\sum_{h = 1}^{H}\V((\check{V}_{h + 1}^k)^2, s_h^k, a_h^k)\eta} + 2H^2\eta\\
&\leq & 4H\sqrt{\sum_{k = 1}^{K}\sum_{h = 1}^{H}\V(\check{V}_{h + 1}^k, s_h^k, a_h^k)\eta} + 2H^2\eta.
\eeqnn
It is easy to check that
\[
\widetilde{V}^k_{h}(s_{h + 1}^k)-P\check{V}^k_{h + 1}(s_h^k, a_h^k) \leq V^k_{h}(s_h^k, a_h^k)-r_h(s_h^k, a_h^k) - PV^k_{h + 1}(s_h^k, a_h^k).
\]
By the same analysis of \eqref{i4_second_upper_bound}, we have
\[
\begin{aligned}
	I_5 &\leq \sum_{k = 1}^{K}\sum_{h = 1}^{H}\Big(P[\check{V}^k_{h + 1}]^2(s_h^k, a_h^k) - \big[\check{V}^k_{h + 1}(s_{h + 1}^k)\big]^2\Big) + 2H\hat{I}_3\\
	&\leq 4H\sqrt{I_5\eta} + 2H^2\eta + 2H\hat{I}_3.
\end{aligned}
\]
Similar to the discussion of (\ref{app.reb.10}), we have that
$I_5 \leq 20H^2\eta + 4H\hat{I}_3.$
\end{proof}
\bmhead{Step 3 : putting them together} Lemma \ref{reb2.1}, Lemma \ref{reb2.4} and Lemma \ref{reb2.5}  show that on the event $\widetilde{\mathcal{E}}_2\cap\widetilde{\mathcal{E}}_3\cap\widetilde{\mathcal{E}}_4$, we have
\[
\begin{aligned}
I_2 &\leq 2\sqrt{I_4\eta} + 2H\eta,\\
I_3 &\leq\hat{I}_3\leq 44\sqrt{I_5C_3|\mathcal{C}_{\sigma}|HD_{L_1,H,\sigma}} + 35\sqrt{I_4C_3|\mathcal{C}_{\sigma}|H\eta_{4}}\\
&\qquad\qquad\qquad\qquad\qquad\qquad\qquad\qquad + H^2|\mathcal{C}_{\sigma}|F_2(\eta, H, \sigma)+232HK\sigma +M_{\sigma
	, \delta}HK\sigma,\\
I_4 &\leq 20H^2\eta + 4H\hat{I}_3 + 4H^2K,\\
I_5 &\leq 20H^2\eta + 4H\hat{I}_3.
\end{aligned}
\]
It follows that
\[
\hat{I}_3 \leq 158\sqrt{\hat{I}_3C_3H^2|\mathcal{C}_{\sigma}|F_1(\eta, H, \sigma)} + G(H, K, \delta, \sigma),
\]
where
\[
\begin{aligned}
G(H, K, \delta, \sigma) :=& 150\sqrt{5C_3H^3|\mathcal{C}_{\sigma}|F_1(\eta, H, \sigma)\eta} +  70\sqrt{C_3H^3K|\mathcal{C}_{\sigma}|\eta_4}\\
&\qquad\qquad\qquad\qquad\qquad+  H^2|\mathcal{C}_{\sigma}|F_2(\eta, H, \sigma) +(M_{\sigma
	, \delta} + 232)HK\sigma.
	\end{aligned}
	\]
	Observe that
	\[
	\begin{aligned}
M_{\sigma, \delta} &= 2L_1 + 2\lambda_r + 2\lambda_p L_1 + D_{\delta}  + 6\widetilde{D}_{L_1, H}\\
&\leq 2L_1 + 2\lambda_r + 2\lambda_p L_1 + 16\sqrt{\frac{C_2\eta_1\eta}{\beta^3}} + \big(C_2\eta_4 + 2C_2\big)\frac{\sigma}{H\beta^2} + \lambda_pL_1\frac{\sigma^2}{KH^2} + \eta_3\\
&\quad+ 12\sqrt{\frac{4C_1\lambda_pL_1\sigma}{\beta KH^2} + \frac{C_1C_2}{\beta^2}} + 6(\eta_5H + 3)\frac{C_2\sigma}{H^2\beta^2} + 6\lambda_pL_1\frac{\sigma^2}{KH^2} + 48\lambda_pL_1\gamma + 72L_1\\
&\leq \widetilde{O}\big(1\big),\\
F_1(\eta, H, \sigma) &= 12D_{L_1,H,\sigma} + 165\eta_4 \leq \widetilde{O}(\widetilde{\mathcal{C}}_{\sigma})\\
F_2(\eta, H, \sigma) &=1+ C_3F_1(\eta, H, \sigma) + 44C_3\lambda_pL_1D_{L_1,H,\sigma} + 35C_3\lambda_pL_1\eta_4\leq \widetilde{O}(\widetilde{\mathcal{C}}_{\sigma}).
\end{aligned}
\]
It follows that $G(H, K, \delta, \sigma)\leq \widetilde{O}\big(HK\sigma + H^2|\mathcal{C}_{\sigma}||\widetilde{\mathcal{C}}_{\sigma}| + \sqrt{H^3K|\mathcal{C}_{\sigma}|}\big)$. Consequently,
\[
\begin{aligned}
I_3\leq \hat{I}_3 &\leq 158\sqrt{\hat{I}_3C_3H^2|\mathcal{C}_{\sigma}|F_1(\eta, H, \sigma)} + G(H, K, \delta, \sigma)\\
&\leq \widetilde{O}\big(HK\sigma + H^2|\mathcal{C}_{\sigma}||\widetilde{\mathcal{C}}_{\sigma}| + \sqrt{H^3K|\mathcal{C}_{\sigma}|}\big),
\end{aligned}
\]
and
\[
\begin{aligned}
I_2 &\leq 2\sqrt{I_4\eta} + 2H\eta \leq \widetilde{O}\big(\sqrt{H^3K|\mathcal{C}_{\sigma}|} + \sqrt{H^2|\mathcal{C}_{\sigma}||\widetilde{\mathcal{C}}_{\sigma}|} +\sqrt{HK\sigma}\big).
\end{aligned}
\]
Since
\[
\pr(\widetilde{\mathcal{E}}_2\cap\widetilde{\mathcal{E}}_3\cap\widetilde{\mathcal{E}}_4) \geq 1 - \bigg(3 + 8|\mathcal{C}_{\sigma^3/(
D_{L_1,F,\sigma} KH^4)}|\bigg)H^2K(\lfloor \log_2 (HK)\rfloor+3)\delta,
\]
we complete the proof of Lemma \ref{reb.2}.\qed

\subsection{Proof of Theorem \ref{thm.1}}
For any $\delta > 0$, from Lemma \ref{goe}, Lemma \ref{opt}, Lemma \ref{concentration}, Lemma \ref{reb.1}, and Lemma \ref{reb.2}, we have
\[
\mathcal{R}(K) \leq I_1 + I_2 + I_3 \leq \widetilde{O}\big(\sqrt{H^3K|\mathcal{C}_{\sigma}|} + H^2|\mathcal{C}_{\sigma}||\widetilde{\mathcal{C}}_{\sigma}| + HK\sigma\big)
\]
holds on events $\mathcal{E}\cap\widetilde{\mathcal{E}}_1\cap\widetilde{\mathcal{E}}_2\cap\widetilde{\mathcal{E}}_3\cap\widetilde{\mathcal{E}}_4$, and
\[
\pr(\mathcal{E}\cap\widetilde{\mathcal{E}}_1\cap\widetilde{\mathcal{E}}_2\cap\widetilde{\mathcal{E}}_3\cap\widetilde{\mathcal{E}}_4) \geq 1 - \delta - \bigg(3 + 10|\mathcal{C}_{\sigma^3/(
D_{L_1,F,\sigma} KH^4)}|\bigg)H^2K(\lfloor \log_2 (HK)\rfloor+3)\delta.
\]
Let $\delta' = \delta + \big(3 + 10|\mathcal{C}_{\sigma^3/(
D_{L_1,F,\sigma} KH^4)}|\big)H^2K(\lfloor \log_2 (HK)\rfloor+3)\delta$. Then
\[
\begin{aligned}
\log\delta' &\sim \log\delta + \log(HK) + \log\big(|\mathcal{C}_{\sigma^3/(
	D_{L_1,F,\sigma} KH^4)}|\big)\\
&\sim \log\delta + \log(HK) + d_{_{SA}}\log (1/\sigma).
\end{aligned}
\]
With probability at least $1 - \delta'$, we have
\[
\mathcal{R}(K) \leq \widetilde{O}\big(\sqrt{H^3K|\mathcal{C}_{\sigma}|} + H^2|\mathcal{C}_{\sigma}||\widetilde{\mathcal{C}}_{\sigma}| + HK\sigma\big).
\]
Denote $\delta'$ as $\delta$ again, we get the desired result.\qed

\section{Auxiliary Lemmas}
\label{app.aux}
We collect
some auxiliary results here. They have been established in previous works and are restated
for the convenience of reference.
\begin{lemma}[Lemma 5 in \cite{Domingues2020KernelBasedRL}]
    \label{aux.6}
    Let $\mathcal{F}_{L, F}$ be a set of $L$-Lipschitz functions from metric space $(\mathcal{S}, \rho_{\mathcal{S}})$ to $[0, F]$. Then its $\varepsilon$-covering number
with respect to the infinity norm is bounded as follows
    \[
        \mathcal{N}(\varepsilon, \mathcal{F}_{L, F}, \Vert\cdot\Vert_{\infty}) \leq \bigg(\frac{8F}{\varepsilon}\bigg)^{\mathcal{N}(\varepsilon/(4L), \mathcal{S}, \rho)}.
    \]
\end{lemma}
\begin{lemma}[Lemma 30 in \cite{chen2021}]
    \label{aux.7}
    For any random variable X satisfying $|X| \leq C$, we have
    \[
        \var[X^2] \leq 4C^2\var[X].
    \]
\end{lemma}
The following lemma is a direct corollary of Lemmas 7 and 8 in \cite{Domingues2020KernelBasedRL} under the setting of the current paper.

\begin{lemma}
    \label{aux.5}
    For $(k, h) \in [K] \times [H]$ and $(s, a) \in \mathcal{S} \times \mathcal{A}$, we have
    \[
        \sum_{l = 1}^{k-1}\widetilde{\mathnormal{w}}_h^{k, l}(s, a)\rho[(s, a), (s_h^l, a_h^l)] \leq 2\sigma\big(1 + \sqrt{\log(k/\beta + e)}\big).
    \]
    In addition, for any real numbers $y_{h, l}$ and $(s', a') \in \mathcal{S} \times \mathcal{A}$,
    \begin{align}
        \label{app.aux.1}
        \sum_{l = 1}^{k-1}[\mathnormal{w}_h^{k, l}(s, a) - \mathnormal{w}_h^{k, l}(s', a')]y_{l} \leq 2C_2k\max_{h, l}|y_{h, l}|\frac{\rho[(s, a), (s', a')]}{\beta\sigma},
    \end{align}
    and
    \begin{align}
        \label{app.aux.2}
        \bigg|\frac{1}{C^k_h(s, a)} - \frac{1}{C^k_h(s', a')}\bigg| \leq C_2k\frac{\rho[(s, a), (s', a')]}{\beta^2\sigma}.
    \end{align}
\end{lemma}

\begin{lemma}[Lemma 9 in \cite{Domingues2020KernelBasedRL}]
    \label{aux.num}
   Let $A_t = \sum_{n = 1}^{t - 1}a_n$, where $\{a_n\}_{n \geq 1}\subset[0,c]$ for some constant $c > 0$. Then for any $b > 0$ and $p > 0$,
    \[
        \sum_{t = 1}^T\frac{a_t}{1 + bA_t} \leq c + \int_{0}^{A_{T + 1} - c}\frac{1}{(1 + bz)^p}dz.
    \]
\end{lemma}

\begin{lemma}[Hoeffding's inequality for martingale]
    \label{aux.1}
    Suppose $X_1,, X_2, \cdots$ is a martingale difference sequence where $|X_i| \leq c \in \mathbb{R}_+$ almost surely. Then for any $n \in \mathbb{N}_+$ and $\delta\in(0, 1)$, we have with probability at least $1 - \delta$,
    \[
        \bigg|\sum_{i = 1}^nX_i\bigg| \leq c\sqrt{2n\log(2/\delta)}.
    \]
\end{lemma}

\section{Experimental Results and Figures}
\label{sec:experimental}
\renewcommand{\thefigure}{\arabic{figure}}
\setcounter{figure}{3}

\begin{figure}[H]
	\centering
	\includegraphics[width=0.8\textwidth]{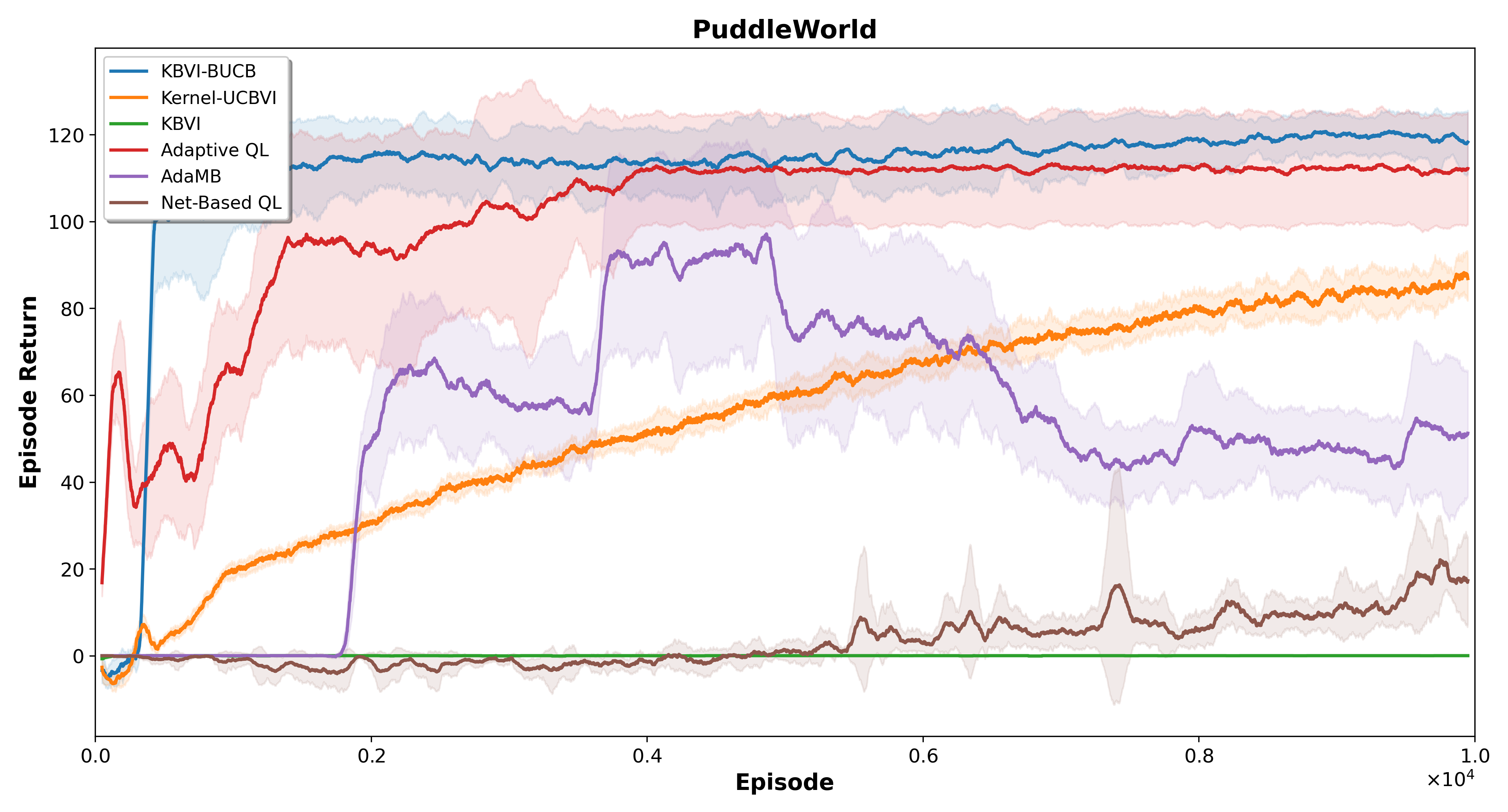}
	\caption{Returns on PuddleWorld with the goal reward set to +10. Solid lines show the mean and shaded regions show one standard deviation across 8 random seeds.}
	\label{fig4}
\end{figure}

We repeat the PuddleWorld experiment with the goal reward reduced from +100 to +10. As shown in Figure \ref{fig4}, KBVI-BUCB still achieves the highest return, confirming that our method is robust to different reward scales.

\begin{figure}[H]
	\centering
	\includegraphics[width=0.8\textwidth]{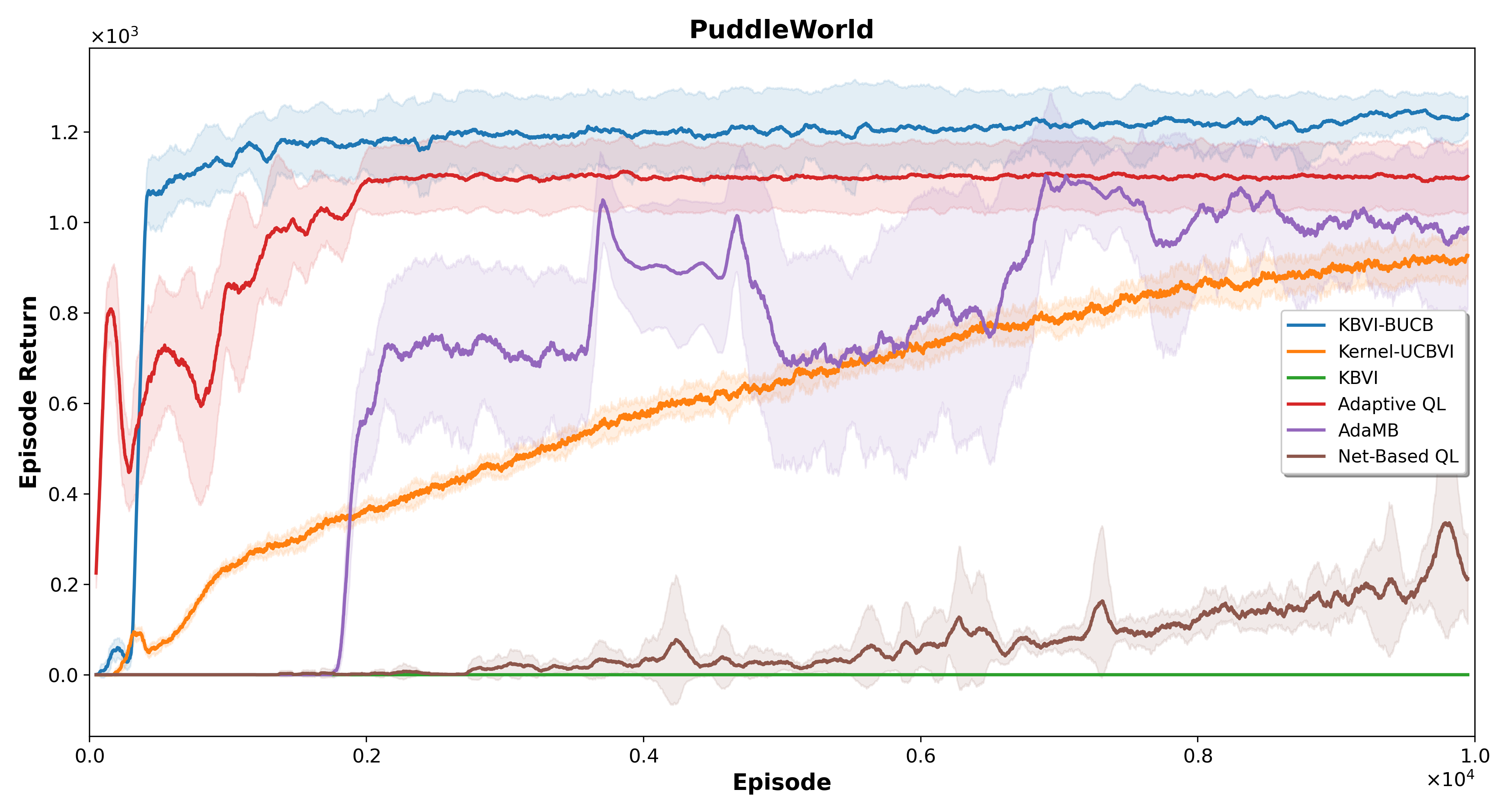}
	\caption{Returns on PuddleWorld with one puddle removed and the other reduced to -1. Solid lines show the mean and shaded regions show one standard deviation across 8 random seeds.}
	\label{fig5}
\end{figure}

We also test a variant where the two puddle regions are modified: the puddle at $[0.6, 0.8] \times [0.2, 0.4]$ is removed, and the puddle at $[0.2, 0.4] \times [0.6, 0.8]$ is reduced from $-10$ to $-1$. The goal reward remains $+100$. This creates an easier environment where only one weak obstacle lies on the path to the goal. As shown in Figure~\ref{fig5}, KBVI-BUCB still achieves the highest return.

\end{appendices}


\bibliography{KUCB.bib}

\end{document}